\documentclass[]{openmoss}
\usepackage{helvet}

\usepackage[toc,page,header]{appendix}
\usepackage{titletoc}
\usepackage{minitoc}

\usepackage{algorithm}
\usepackage{algpseudocode}

\usepackage{wasysym}
\usepackage{fontawesome}

\usepackage{enumitem}

\usepackage{listings}

\usepackage{pgfplots}
\pgfplotsset{compat=1.18}
\usetikzlibrary{arrows.meta, positioning, calc, fit, backgrounds, shapes.geometric, decorations.pathmorphing}

\usepackage{arydshln}
\usepackage{adjustbox}
\usepackage{colortbl}
\usepackage{multirow}

\usepackage{xspace}
\usepackage{nicefrac}
\usepackage{float}
\usepackage{comment}
\usepackage{wrapfig}
\usepackage{pifont}
\usepackage{dsfont}

\usepackage{amsmath, amssymb}
\usepackage{amsthm}
\newtheorem{definition}{Definition}
\newtheorem{proposition}{Proposition}

\newcommand{\bench}{\textsc{SciJudgeBench}\xspace}
\newcommand{\model}{\textsc{Scientific Judge}\xspace}
\newcommand{\judge}{\textsc{Scientific Judge}\xspace}
\newcommand{\judgetable}{\textsc{SciJudge}\xspace}

\newcommand{\thinker}{\textsc{Scientific Thinker}\xspace}
\newcommand{\thinkertable}{\textsc{SciThinker}\xspace}
\newcommand{\gain}[1]{{\scriptsize\color{green!70!black}(+#1)}}
\newcommand{\loss}[1]{{\scriptsize\color{red!70!black}(#1)}}

\usepackage{tabularx}
\usepackage{xcolor}
\definecolor{case_red}{RGB}{190,0,0}
\definecolor{case_blue}{RGB}{0,102,204}
\definecolor{case_green}{RGB}{0,153,0}

\title{AI Can Learn Scientific Taste}

\author{
    Jingqi Tong\textsuperscript{1,2,3,*},
    Mingzhe Li\textsuperscript{1,2,3,*},
    Hangcheng Li\textsuperscript{1,2,3,*},
    Yongzhuo Yang\textsuperscript{1,3,$\ddagger$},
    Yurong Mou\textsuperscript{1,2,3,$\ddagger$}, \\
    Weijie Ma\textsuperscript{1,2},
    Hongji Chen\textsuperscript{1,3},
    Xiaoran Liu\textsuperscript{1,2,3},
    Qinyuan Cheng\textsuperscript{1,2,3},
    Ming Zhang\textsuperscript{1},
    Qiguang Chen\textsuperscript{5}, \\
    Weifeng Ge\textsuperscript{1},
    Qipeng Guo\textsuperscript{2},
    Tianlei Ying\textsuperscript{1,2},
    Tianxiang Sun\textsuperscript{2},
    Yining Zheng\textsuperscript{1,2,3},
    Zhiheng Xi\textsuperscript{1}, \\
    Xinchi Chen\textsuperscript{1,3,$\dagger$},
    Jun Zhao\textsuperscript{1,$\dagger$},
    Ning Ding\textsuperscript{4},
    Xuanjing Huang\textsuperscript{1},
    Yu-Gang Jiang\textsuperscript{1},
    Xipeng Qiu\textsuperscript{1,2,3,$\dagger$}
}
\affiliation[1]{\mbox{Fudan University}}
\affiliation[2]{\mbox{Shanghai Innovation Institute}}
\affiliation[3]{\mbox{OpenMOSS Team}}
\affiliation[4]{\mbox{Tsinghua University}}
\affiliation[5]{\mbox{Central South University}}

\contribution{* Equal contribution. $\dagger$ Corresponding author. $\ddagger$ Core contributors}

\abstract{
Scientific discovery depends on expert judgement and foresight, which we call scientific taste: the ability to judge and propose research ideas with the potential for long-term scientific impact.
Scientific taste is largely concentrated among highly experienced researchers, whose expertise is usually limited to a few specialised fields. If AI could learn scientific taste, it could reduce reliance on human experts and accelerate scientific discovery.
Whether AI can learn this ability remains an open question. We introduce Reinforcement Learning from Community Feedback (RLCF) to learn judgement and ideation. \judge learns from community feedback, such as citations. \thinker learns to propose research ideas with high potential impact.
Experiments show that \judge outperforms strong LLM baselines and that learned judgement generalises to future-year papers, other community metrics, and unseen fields. Furthermore, \thinker proposes research ideas with higher potential impact than those proposed by baselines. These results suggest that AI can learn scientific taste, marking an important step towards AI systems that could help accelerate scientific discovery.
}

\contact{\email{jqtong25@m.fudan.edu.cn}, \email{\{xc\_chen,zhaoj19,xpqiu\}@fudan.edu.cn}}
\coderepository{\url{https://github.com/tongjingqi/AI-Can-Learn-Scientific-Taste}}

\begin{document}
\maketitle
\vspace{-6mm}

\begin{figure}[!htbp]
  \centering
  \includegraphics[width=\textwidth]{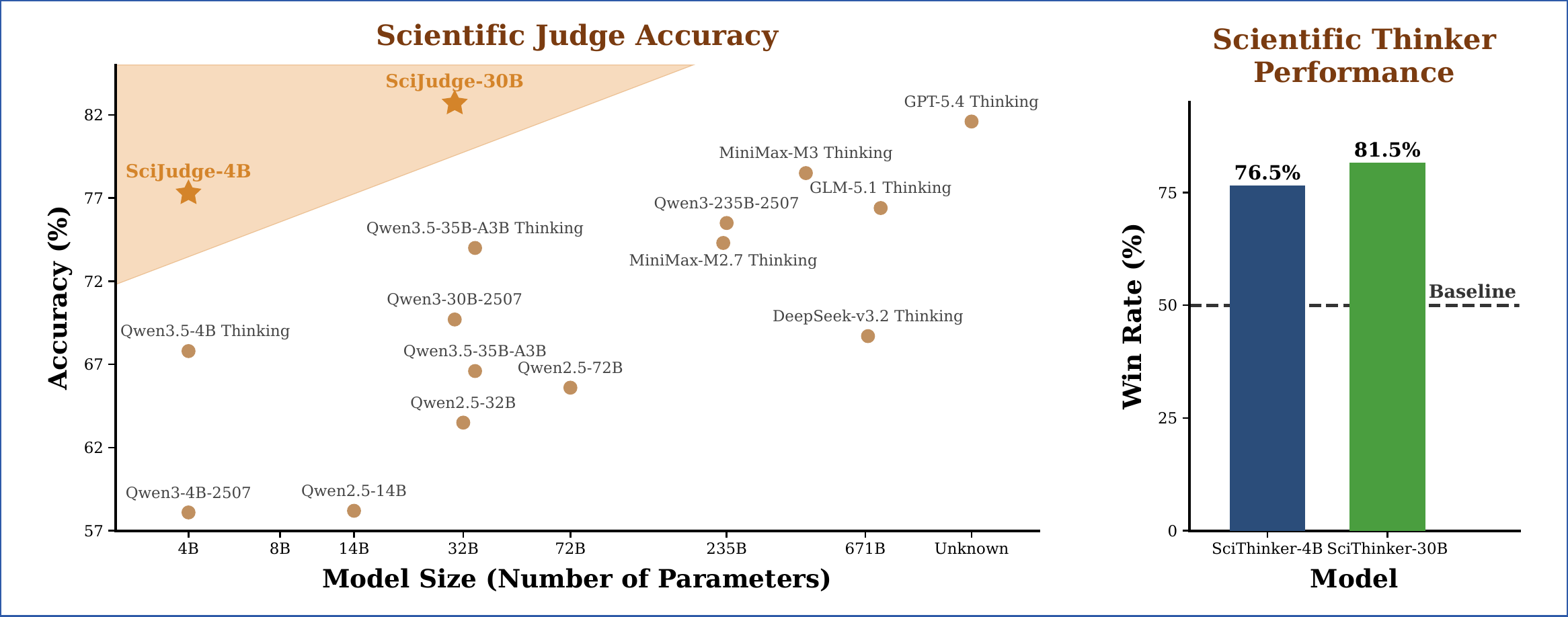}
  \caption{\textbf{(Left)} \model accuracy on \bench, comparing trained models (stars) with strong LLMs. \textbf{(Right)} In-domain \thinker win rates over base policies under ensemble evaluation.}
  \label{fig:performance_teaser}
\end{figure}

% \newpage
\section{Introduction}
\label{sec:intro}

Scientific discovery depends on expert judgement and foresight. When reading a new result, an experienced researcher can judge whether it represents a genuine advance and identify promising follow-up directions---often long before the broader community has recognised the work's significance. This ability to judge and propose research ideas with the potential for long-term scientific impact is what we call \textit{scientific taste}~\citep{tjian2023terri,mitchison2013question}.

Scientific taste is not simply a matter of subjective preference. Hume argued that a standard of taste can emerge from the joint verdict of qualified judges rather than from arbitrary individual preferences~\citep{hume1757standard}. Kant~\citep{kant1790critique} introduced taste as a kind of ``sensus communis'', a shared sense that considers how others might judge rather than merely reflecting personal preference.

In science, such a community verdict is reflected in long-term interactions among researchers. Work that aligns with this scientific taste is more likely to be reused and extended by subsequent studies. Ultimately, this community feedback is expressed through signals such as citations, one of the most common measures of scientific impact~\citep{wang2013quantifying,fortunato2018science}.

In practice, such taste is scarce. It depends on knowledge, skill, and sustained research experience~\citep{ericsson1993role,brownstein2019expert}. Even experienced researchers tend to specialise in a limited number of areas as the burden of knowledge grows~\citep{jones2009burden}. If AI could learn this taste, it could reduce reliance on a small group of human experts and speed up discovery.

However, whether AI can learn this broader ability remains an open question~\citep{wang2024scientific,si2025ideationexecutiongapexecutionoutcomes}. We propose Reinforcement Learning from Community Feedback (RLCF) to learn scientific judgement and ideation. \judge learns from community feedback, such as citations: it compares two papers using its own evaluation rubric and reasoning before choosing the better one. \thinker learns to propose scientific ideas that are guided by this judgement, have high academic value and potential impact, and align with community preferences (Figure~\ref{fig:main}).

\judge outperforms strong LLM baselines on \bench, while learned judgement generalises to future-year papers, other community metrics, and unseen fields, suggesting that learned judgement is transferable. Moreover, experiments show that \thinker proposes ideas judged to have higher potential impact than those proposed by baselines. Together, our results suggest that scientific taste can be learned from community feedback, marking a step towards AI systems that could help accelerate scientific discovery.

% This paper makes the following contributions:

% \begin{itemize}
% \item We formulate scientific taste learning as a preference modeling and alignment problem, proposing the Reinforcement Learning from Community Feedback (RLCF) training paradigm which leverages large-scale community signals (e.g., citations) as supervision.
% \item We construct \bench for training and evaluating AI's scientific judgement, which consists of 720K field- and time-matched citation-based paper abstract pairs.
% \item We train \judge for scientific judgement, with the 30B variant outperforming the evaluated strong LLM baselines in average accuracy on the main test set. Learned judgement generalises to future-year papers, other community metrics, and unseen fields, and its gains persist under author/institution and topic controls. We further train \thinker for ideation, which proposes ideas with higher potential impact after training.
% \item Our findings suggest that AI models can learn scientific taste from large-scale community feedback.
% \end{itemize}

\begin{table}[!t]
    \small
    \vspace{0.5em}
    \noindent
    \begin{tabularx}{\textwidth}{@{}X@{}}
    \toprule
    \textbf{Prompt:} You are an AI researcher. Given the following paper title and abstract, propose a follow-up idea with high potential impact. \par
    \vspace{0.2em}
    Title: The Invisible Leash: Why RLVR May or May Not Escape Its Origin~\citep{wu2025invisible} \par
    Abstract: ... \par
    \vspace{0.3em} \\
    \midrule
    <think> \par
    ...The original paper discusses... \par
    So my idea: ...during RLVR training, we'd \textcolor{case_green}{\textbf{add a bonus for trajectories that deviate from the base model's high-probability paths but still have high potential}} (measured by... maybe a confidence threshold?)...
    
    Testing academic value:... Title brainstorming:... Abstract structure must include... Okay, time to write it cleanly. \par
    </think> \par
    Title: Beyond the Leash: Uncertainty-Guided Exploration for Verifiable Reward Learning \par
    Abstract: ...We propose..., a framework...\par
    \vspace{0.2em} \\
    \bottomrule
    \end{tabularx}
    \caption{Task of \thinker, which proposes a follow-up research idea given a paper (title and abstract). Through step-by-step thinking, it extracts the problem about RLVR pointed out by the given paper, proposes a specific algorithmic solution, and finally gives a title and abstract. Full prompt is in Appendix~\ref{sec:app_thinker_prompt}. Full research idea and comparison with the model before training are presented in Appendix~\ref{sec:app_thinker_examples}.}
\end{table}

\section{Background and Related Work}
\label{sec:related}

\subsection{Definition of Scientific Taste}

Scientific discovery depends on expert judgement and foresight, which we call scientific taste: the ability to \textbf{judge} and \textbf{propose} research ideas with potential for long-term scientific impact. To make this notion precise, we provide a layered formal definition.

\paragraph{Potential Impact.}
We first formalize what it means for a research idea to have potential impact. Citations are the most common way to measure the impact of scientific research \citep{fortunato2018science,wang2013quantifying}. Consider a published paper $p$. Let $c_t(p)$ be the number of new citations that paper $p$ receives in year $t$ after publication. We model $c_t(p)$ as a non-negative random variable drawn from a distribution that depends on the paper and its temporal context. The cumulative expected impact of paper $p$ is defined as:
\begin{equation}
I(p) = \lim_{N \to \infty} \sum_{t=1}^{N} \mathbb{E}[c_t(p)],
\end{equation}
where $\mathbb{E}[c_t(p)]$ denotes the expected citation increment in year $t$. A paper with a larger $I(p)$ is considered to have higher potential impact. 

%Notably, relative comparison remains well-defined even when both series diverge

\paragraph{Judgement Capability.} The judgement capability of a model $\theta$ is measured by the expected accuracy of comparing the cumulative expected impact of paper pairs. Let $\mathcal{D}$ denote a distribution over field- and time-matched paper pairs. For a single pair $(p_a, p_b)$, the ground-truth label is:
\begin{equation}
y(p_a, p_b) = 
\begin{cases}
1, & \text{if } I(p_a) > I(p_b), \\
0, & \text{otherwise}.
\end{cases}
\end{equation}

Note that this label is well-defined even when both $I(p_a)$ and $I(p_b)$ diverge (see Appendix~\ref{app:divergent_comparison} for a formal proof). In practice,  we work with finite-horizon approximations $I_N(p) = \sum_{t=1}^{N} \mathbb{E}[c_t(p)]$.
%In practice, since $I(p)$ may diverge, we work with finite-horizon approximations $I_N(p) = \sum_{t=1}^{N} \mathbb{E}[c_t(p)]$ or use relative comparisons between papers rather than absolute values (see Appendix~\ref{app:divergent_comparison} for a formal proof).

The judgement capability is:
\begin{equation}
\textsc{JudgeCap}(\theta) = \mathbb{E}_{(p_a, p_b) \sim \mathcal{D}} \Big[ \mathds{1}\!\left[\textsc{Judge}_{\theta}(p_a, p_b) = y(p_a, p_b)\right] \Big],
\end{equation}
where $\textsc{Judge}_{\theta}(p_a, p_b)$ is the model's predicted result. A higher $\textsc{JudgeCap}(\theta)$ indicates stronger judgement capability.

\paragraph{Ideation Capability.} The ideation capability of a model $\phi$ is characterized by the expected impact of the ideas it proposes. Given a seed reference paper $s \in \mathcal{S}$, model $\phi$ generates a new research idea $\textsc{Thinker}_{\phi}(s)$. The ideation capability is:
\begin{equation}
\textsc{ThinkerCap}(\phi) = \mathbb{E}_{s \sim \mathcal{S}} \left[ I\!\left(\textsc{Thinker}_{\phi}(s)\right) \right].
\end{equation}
For two models $\phi_A$ and $\phi_B$, we say $\phi_A$ has stronger ideation capability than $\phi_B$ if $\textsc{ThinkerCap}(\phi_A) > \textsc{ThinkerCap}(\phi_B)$.

\paragraph{Scientific Taste.}
We refer to the combination of judgement capability and ideation capability as scientific taste. Formally, a model possesses strong scientific taste if it achieves both high $\textsc{JudgeCap}$ and high $\textsc{ThinkerCap}$.

% \paragraph{Scientific Taste.}
% We refer to the combination of judgement capability and ideation capability as scientific taste. Formally, the scientific taste of a model $\psi$ is defined as:
% \begin{equation}
% \textsc{Taste}(\theta,\phi) = \textsc{JudgeCap}(\theta) + \beta \cdot \sigma\!\left(\textsc{ThinkerCap}(\phi)\right),
% \end{equation}
% where $\sigma(\cdot)$ is a normalization function that maps the raw expected impact to $[0, 1]$, making it align with $\textsc{JudgeCap}(\psi) \in [0, 1]$, and $\beta > 0$ is a weighting coefficient that controls the relative importance of ideation capability versus judgement capability.

\begin{table}[!tbp]
    \small
    \vspace{0.5em}
    \noindent
    \begin{tabularx}{\textwidth}{@{}X@{}}
    \toprule
    \textbf{Prompt:} Based on the titles, abstracts, and publication dates of the following two papers A and B, determine which paper has a higher citation count. \par
    \vspace{0.2em}
    \textbf{Paper A:} Group Sequence Policy Optimization~\citep{xu2025gspo} (Published: 2025-07-24) \textbf{Abstract:} ... \par
    \vspace{0.2em}
    \textbf{Paper B:} Self-Foveate: Enhancing Diversity and Difficulty of Synthesized Instructions from Unsupervised Text via Multi-Level Foveation~\citep{li2025selffoveate} (Published: 2025-07-31) \textbf{Abstract:} ... \par
    \vspace{0.2em} \\
    \midrule
    \vspace{0.2em}
    \textbf{Reasoning:} To determine which paper has a higher citation count, we need to consider several factors... 1. Publication Date... 2. Topic and Relevance... 3. Impact and Practical Relevance... 4. Broader Research Dynamics... 5. Citation Momentum... Given that both papers are very recent... and assuming they are both well-received... \textcolor{case_green}{the more enduring nature of optimization algorithms in a rapidly evolving field where data synthesis pipelines change frequently} give Paper A a stronger potential for higher citation count. \par
    \vspace{0.5em}
    \textbf{Answer:} A \par
    \vspace{0.2em} \\
    \bottomrule
    \end{tabularx}
    \caption{Case of our 30B \judge (i.e., \judgetable-30B), which correctly predicts that Paper A will receive more citations by reasoning about topic generality, long-term impact, and citation dynamics. The full case is provided in Appendix~\ref{sec:app_examples}.}
\end{table}

\subsection{AI for Scientific Research}

Current training for AI Scientists is mainly targeting literature search~\citep{openai_deep_research, jin2025search,zheng2025deepresearcher,zhang2026deepresearchagentsretrieve} and experiment execution~\citep{chan2024mle, lu2024ai_scientist,survey2025ai_scientists,weng2025deepscientist,wu2025innovatorbench,schmidgall2025agentlaboratoryusingllm,yamada2025aiscientistv2workshoplevelautomated}. However, these capabilities address how to carry out research rather than what research directions are worth pursuing. Human evaluations show that while LLMs can generate novel research ideas, they often struggle to reliably distinguish potentially high-impact directions from ideas that are superficially novel but trivial~\citep{wang2024scientific}. This gap constitutes a key difference between today's AI Scientists and human experts, which we refer to as scientific taste, including (1) judging the scientific value of candidate ideas, and (2) proposing research questions, hypotheses, and methods with high potential impact.

Recent studies have explored leveraging LLMs to evaluate academic manuscripts, predict review scores, and generate feedback~\citep{jin2024agentreview,d2024marg,zhang2025aixiv,liang2024can,thakkar2025can,gottweis2025towards}. However, these works primarily employ language models as components in review pipelines, rather than enhancing the model's intrinsic capability for scientific judgment.
Prior works \citep{zhu2025deepreview,weng2024cycleresearcher} typically use supervised fine-tuning (SFT) to train models on reviewer feedback, whereas we use community feedback through reinforcement learning to train models to judge and propose ideas with high potential impact, aligning it more closely with broader community preferences.

Current ideation methods also exhibit clear limitations. In practice, ideation improvement is frequently driven by random heuristics or simple brainstorming strategies~\citep{wang2024scientific}. Recent work such as OpenNovelty uses information retrieval to measure how different an idea is from prior work (i.e., novelty)~\citep{zhang2026opennoveltyllmpoweredagenticverifiable}. Currently, optimization of ideation is primarily focused on external retrieval and model prompt stimulation~\citep{zhang2026opennoveltyllmpoweredagenticverifiable,wang2024scientific}, while enhancing the model's intrinsic ideation capabilities remains underexplored.

\subsection{RL Training Paradigms for LLMs}

Reinforcement learning can be used to improve alignment~\citep{ouyang2022training}. Reinforcement Learning from Human Feedback (RLHF)~\citep{ouyang2022training,stiennon2020learning,bai2022hh} collects human preference annotations, trains a reward model to capture human preferences, and then optimizes a policy model with that reward, enabling better alignment to subjective preferences such as being helpful and harmless. Recent efforts further scale reward modeling and develop standardized benchmarks for evaluating reward models~\citep{worldpm,rewardbench,rmb}. For tasks such as math and coding, Reinforcement Learning with Verifiable Reward (RLVR)~\citep{deepseek2025r1,shao2024deepseekmathpushinglimitsmathematical} instead leverages verifiable rewards provided by ground-truth answers, unit tests, or formal checkers, and has led to large gains in mathematical reasoning, code generation, and broader post-training pipelines~\citep{tulu3,game_rl,mathtrap}.

However, RLVR is inherently tied to tasks with verifiable ground-truth, making it difficult to apply to open-ended tasks such as scientific judging and idea generation~\citep{deepseek2025r1}. RLHF, on the other hand, is limited by its reliance on costly human annotations~\citep{ouyang2022training,stiennon2020learning} and inability to reflect community-level preferences through individual preferences alone. Our work proposes Reinforcement Learning from Community Feedback (RLCF), leveraging scalable community feedback signals which naturally emerge from community interactions, thereby inherently capturing community preferences.

\begin{figure}[t]
  \centering
  \includegraphics[width=\textwidth]{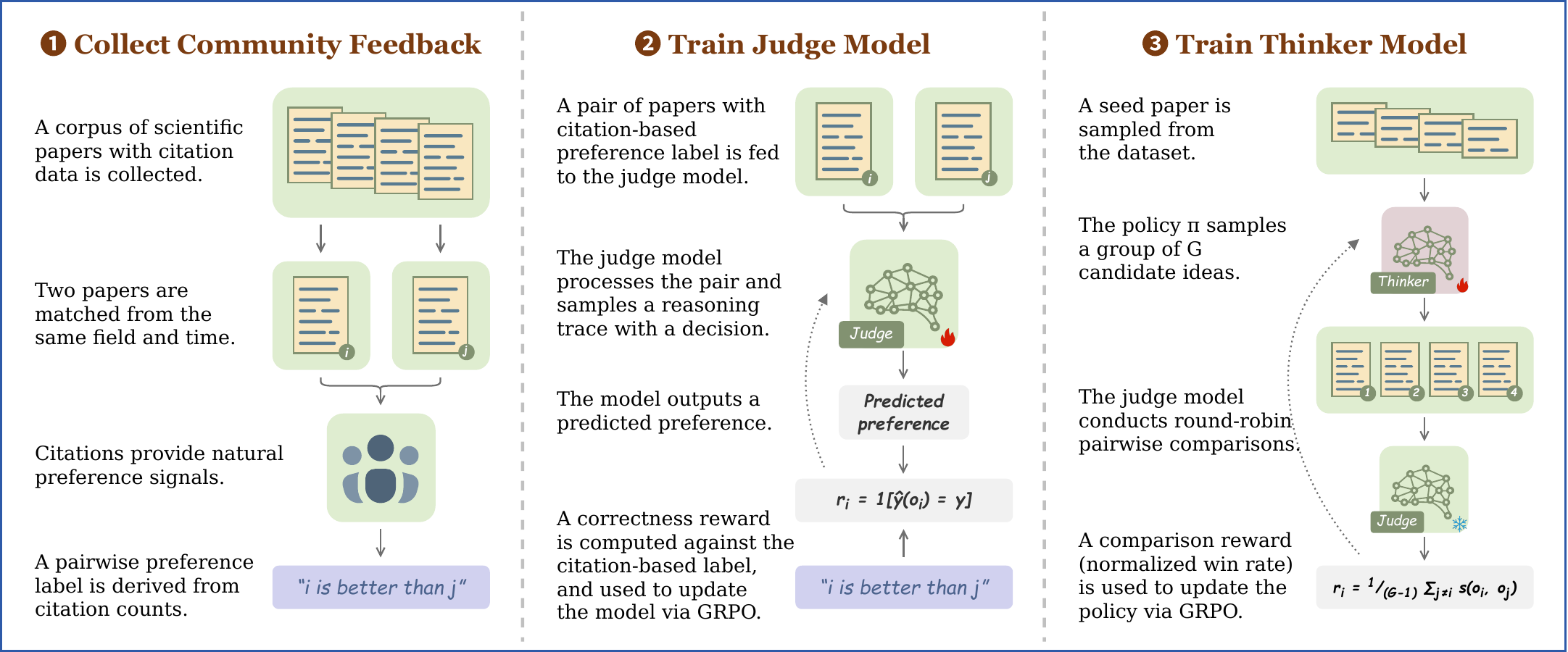}
  \caption{Overview of Reinforcement Learning from Community Feedback (RLCF).
\textbf{(1)} Community feedback is collected as pairwise preference signals
from naturally occurring community behavior.
\textbf{(2)} A preference model is trained via GRPO to predict which item
in a pair receives stronger community reception.
\textbf{(3)} A policy model is trained via comparison-based GRPO: for each
input, the policy samples a group of outputs, the preference model conducts
pairwise comparisons to produce scalar rewards, and the policy is updated
accordingly. In this work, we instantiate RLCF for scientific taste learning,
where community feedback is derived from citation signals.}
  \label{fig:main}
\end{figure}

\section{Reinforcement Learning from Community Feedback}
\label{sec:method}

To learn scientific taste, we introduce Reinforcement Learning from Community Feedback (RLCF), a training
paradigm that uses large-scale community signals as supervision. RLCF proceeds in three stages:
(1)~\textbf{construct community preference}, where we collect a community feedback signal to construct community preference data;
(2)~\textbf{preference modeling},
where we train \judge to predict potential impact of research ideas; and
(3)~\textbf{preference alignment}, where we use \judge as a reward model to supervise \thinker to generate scientific ideas with high potential impact.

\subsection{Community Feedback as Supervision}

We use citations as scientific community feedback signals, because citation count is a community verdict reflected through long-term interactions within a research community. High citation counts can indicate high scientific impact~\citep{zhao2025words}.
To mitigate field and time biases in raw citation counts, we construct training data by pairing articles from the same field and year, where the one with significantly more citations serves as the preferred (higher-impact) item.

Each training example consists of two scientific ideas represented by their titles and abstracts~\citep{zhao2025words, zhao2025naipv2debiasedpairwiselearning}, with a binary label indicating which one has higher relative citations.
We refer to the resulting dataset as \bench, which transforms community feedback into pairwise supervision signals, enabling scalable preference learning.

\subsection{Preference Modeling: \judge}
\label{sec:method_critic}
%It is also a generative reward model that produces a preference decision through reasoning.
\judge predicts which research idea has higher potential impact from pairwise comparisons. We train \judge through reinforcement learning on the training set of \bench, using Group Relative Policy Optimization (GRPO)~\citep{shao2024deepseekmathpushinglimitsmathematical}. For each input $x$, the policy $\pi_\theta$ samples a group of $G$ outputs $\{o_i\}_{i=1}^G$, each consisting of a reasoning trace and a preference prediction. The reward is a binary correctness signal:
\begin{equation}
\label{eq:reward}
r_i =
\begin{cases}
1, & \text{if } \hat{y}(o_i) = y, \\
0, & \text{otherwise.}
\end{cases}
\end{equation}
where $\hat{y}(o_i)$ extracts the predicted preference from output $o_i$ and $y$ is the observed label. Within each group, advantages are normalized:
$\hat{A}_i = (r_i - \operatorname{mean}(\mathbf{r})) /
\operatorname{std}(\mathbf{r})$.
The policy is updated by maximizing a clipped surrogate objective with a KL penalty toward a reference policy $\pi_{\mathrm{ref}}$:
\begin{equation}
\label{eq:grpo}
  \mathcal{J}(\theta) = \mathbb{E}_{x}\!\left[\,
    \frac{1}{G}\sum_{i=1}^{G}
      \min\!\Big(\rho_i\,\hat{A}_i,\;
        \operatorname{clip}\!\big(\rho_i,\,1{-}\epsilon,\,1{+}\epsilon\big)\,
        \hat{A}_i\Big)
    \;-\; \beta\,D_{\mathrm{KL}}\!\big(\pi_\theta \,\|\,
      \pi_{\mathrm{ref}}\big)
  \,\right],
\end{equation}
where $\rho_i = \pi_\theta(o_i \mid x)\,/\,\pi_{\mathrm{old}}(o_i \mid x)$ is the importance ratio, $\epsilon$ is the clipping range, and $\beta$ controls the strength of the KL penalty. Hyperparameter values are provided in Appendix~\ref{sec:app_training}.

\subsection{Preference Alignment: \thinker}
\label{sec:method_thinker}

We use \judge as a generative reward model to train \thinker, a policy model which learns to propose scientific ideas with high potential impact. This is an open-ended task with no ground-truth labels, and scoring a single scientific idea is difficult due to the lack of an objective and universal criterion. However, pairwise comparison is more natural and reliable, because it is easier to compare two ideas. We therefore design Comparison-Based GRPO~\citep{guo2025reward,zhang2026arenarlscalingrlopenended,wang2025pref}, using pairwise preferences from \judge to compute each idea's win rate within a group as the reward.

\paragraph{Comparison-Based GRPO.}
Given a prompt $x$ containing a seed paper, the policy $\pi_\theta$ samples a group of $G$ responses $\{o_1, \ldots, o_G\}$, each providing a candidate research idea. Instead of directly scoring each idea, we conduct a round-robin tournament judged by the reward model. Each candidate idea is compared with all the others by \judge, producing a total of $\binom{G}{2}$ pairwise comparison results. The comparison-based reward for $o_i$ is the research idea's win rate within the group:
\begin{equation}
\label{eq:comparison_reward}
  r_i \;=\; \frac{1}{G-1}\sum_{j \neq i} s(o_i, o_j),
\end{equation}
where $s(o_i, o_j) \in \{0, 1\}$ denotes whether the research idea of $o_i$ wins against that of $o_j$ under the reward model's judgement: $s(o_i, o_j)=1$ if it wins, and $0$ otherwise. Given these rewards, the training objective is the same as vanilla GRPO (Eq.~\ref{eq:grpo}).

In summary, Comparison-Based GRPO leverages comparison between sampled responses to calculate rewards, making it suitable for open-ended tasks, such as scientific ideation.

\begin{figure}[t]
  \centering
  \includegraphics[width=\columnwidth]{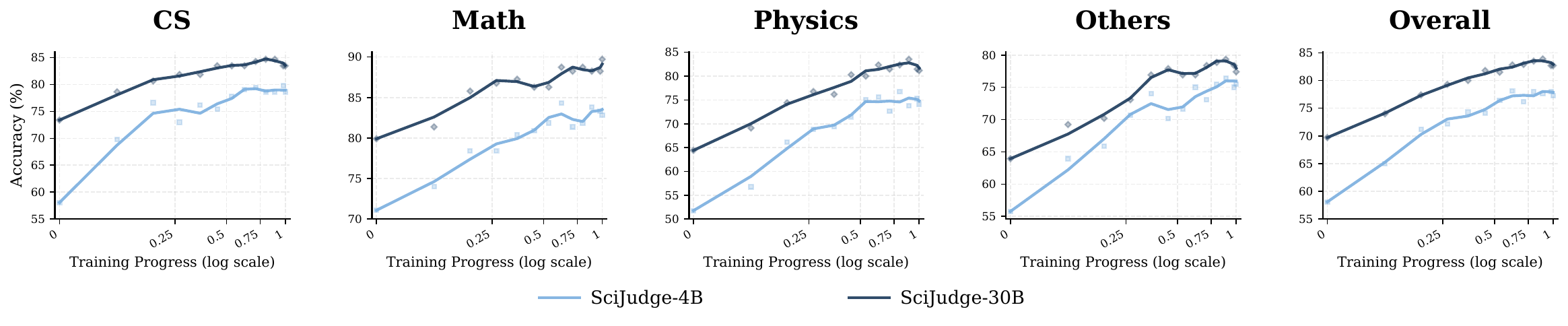}
  \vspace{-2em}
  \caption{Scaling performance of \judge on \bench (in-domain). Both \judgetable-Qwen3-4B and \judgetable-Qwen3-30B improve consistently across categories throughout training.}
  \label{fig:scaling_curve}
\end{figure}

\section{AI Can Learn Scientific Judgement}
\label{sec:exp_judge}

In this section, we focus on training \judge. We first establish the scaling trend of scientific judgement training (\S\ref{sec:exp_main}), then verify that the learned scientific judgement generalizes across time, fields, and other community metrics, and that its gains persist under controlled comparisons (\S\ref{sec:exp_ood}).

\subsection{Experimental Setup}
\label{sec:exp_setup}

\begin{table*}[!t]
  \centering
  \small
  \setlength{\tabcolsep}{4pt}
  \renewcommand{\arraystretch}{1.2}
  \begin{tabular}{l*{5}{>{\centering\arraybackslash}p{2.0cm}}}
    \toprule
    \textbf{Model} & \textbf{CS} & \textbf{Math} & \textbf{Physics} & \textbf{Others} & \textbf{Overall} \\
    \midrule
    \multicolumn{6}{c}{\textit{Open-source Models}} \\
    \midrule
    Qwen3-4B-Instruct & 58.1 & 71.1 & 51.8 & 55.8 & 58.1 \\
    \rowcolor{gray!10} \judgetable-Qwen3-4B & 78.6 \gain{20.6} & 82.8 \gain{11.8} & 74.1 \gain{22.4} & 75.5 \gain{19.7} & 77.3 \gain{19.2} \\
    Qwen3-30B-A3B-Instruct & 73.4 & 79.9 & 64.4 & 63.9 & 69.7 \\
    \rowcolor{gray!10} \judgetable-Qwen3-30B & 83.5 \gain{10.1} & 89.7 \gain{9.8} & 81.2 \gain{16.8} & 77.4 \gain{13.5} & 82.7 \gain{13.0} \\
    Qwen2.5-1.5B-Instruct & 4.0 & 6.4 & 5.3 & 5.8 & 5.3 \\
    \rowcolor{gray!10} \judgetable-Qwen2.5-1.5B & 60.1 \gain{56.0} & 66.7 \gain{60.3} & 60.3 \gain{55.0} & 62.0 \gain{56.2} & 61.9 \gain{56.6} \\
    Qwen2.5-3B-Instruct & 22.2 & 36.3 & 27.1 & 28.4 & 28.0 \\
    \rowcolor{gray!10} \judgetable-Qwen2.5-3B & 73.4 \gain{51.2} & 80.4 \gain{44.1} & 70.6 \gain{43.5} & 68.3 \gain{39.9} & 72.8 \gain{44.8} \\
    Qwen2.5-7B-Instruct & 48.0 & 53.9 & 48.8 & 51.4 & 50.2 \\
    \rowcolor{gray!10} \judgetable-Qwen2.5-7B & 75.8 \gain{27.8} & 85.8 \gain{31.9} & 73.2 \gain{24.4} & 74.0 \gain{22.6} & 76.6 \gain{26.4} \\
    Qwen2.5-14B-Instruct & 53.2 & 70.1 & 56.2 & 55.8 & 58.2 \\
    \rowcolor{gray!10} \judgetable-Qwen2.5-14B & 80.2 \gain{27.0} & 86.8 \gain{16.7} & 77.1 \gain{20.9} & 77.4 \gain{21.6} & 79.9 \gain{21.7} \\
    Qwen2.5-32B-Instruct & 62.5 & 73.5 & 60.3 & 60.1 & 63.5 \\
    \rowcolor{gray!10} \judgetable-Qwen2.5-32B & 82.7 \gain{20.2} & 86.8 \gain{13.2} & 79.4 \gain{19.1} & 78.4 \gain{18.3} & 81.5 \gain{18.0} \\
    Llama3.1-8B-Instruct & 47.6 & 53.9 & 45.3 & 38.9 & 46.3 \\
    \rowcolor{gray!10} \judgetable-Llama3.1-8B & 58.9 \gain{11.3} & 62.8 \gain{8.8} & 54.7 \gain{9.4} & 57.2 \gain{18.3} & 57.9 \gain{11.6} \\
    \midrule
    \multicolumn{6}{c}{\textit{Strong LLM Baselines}} \\
    \midrule
    GPT-5.4 Thinking & 87.5 & 84.8 & 77.4 & 78.4 & 81.6 \\
    MiniMax-M3 Thinking & 86.7 & 83.8 & 72.9 & 72.6 & 78.5 \\
    GLM-5.1 Thinking & 84.3 & 80.9 & 68.8 & 75.0 & 76.4 \\
    DeepSeek-v3.2 Thinking & 73.0 & 76.0 & 63.2 & 65.4 & 68.7 \\
    \bottomrule
  \end{tabular}
  \caption{Main results on \bench (in-domain test set). We report pairwise accuracy (\%) with position-swap consistency for predicting which paper has higher citations. Scores and gains are rounded to one decimal place; gains are computed from unrounded accuracies.}
  \label{tab:exp_main}
\end{table*}

\paragraph{Training Data.}
We construct \bench from 2.1M arXiv papers published through 2024, yielding 720,341 field- and time-matched preference pairs across Computer Science, Mathematics, Physics, and other fields. Preference labels are derived from citation counts. See Appendix~\ref{sec:app_dataset} for construction details.

\paragraph{Test Sets.}
We evaluate \judge under complementary in-domain, temporal, metric-transfer, and controlled-comparison settings. The citation-based settings match pairs by field and publication time, so comparisons involve papers from similar areas and periods. \textbf{(1) Main (In-domain):} 1,000 citation-preference pairs sampled from an 8,830-pair source pool and stratified across CS, Physics, Math, and Others. This setting measures in-distribution citation-preference judgement across major scientific fields. \textbf{(2) Temporal OOD:} 904 pairs from papers published in 2025, after the training period, testing whether learned citation preferences extrapolate to future papers. Both citation-based test sets are detailed in Appendix~\ref{sec:app_citation_tests}. \textbf{(3) Metric OOD:} We use 611 review-score pairs drawn from 2017--2026 submissions to the International Conference on Learning Representations (ICLR), a peer-reviewed machine-learning conference. We also use 599 arXiv paper pairs labelled by Altmetric attention scores, which summarize online attention to research outputs (Appendix~\ref{sec:app_metric_tests}). \textbf{(4) Controlled comparisons:} 541 citation-preference pairs linked to the same CSRankings faculty member and matched by publication quarter and local subcategory, and 245 embedding-matched topic-control pairs (Appendix~\ref{sec:app_controlled_tests}). For field generalization, we separately train CS-only variants and evaluate them across the main-test fields. An additional evaluation on 160 bioRxiv biology pairs is provided in Appendix~\ref{sec:app_bio_ood}. See Appendix~\ref{sec:app_dataset} for the complete field-to-subcategory mapping.

\paragraph{Models.}
We train \judge (\judgetable for short) on the Qwen2.5-Instruct series (1.5B, 3B, 7B, 14B, 32B parameters)~\citep{qwen2.5}, Qwen3-4B-Instruct-2507, Qwen3-30B-A3B-Instruct-2507~\citep{yang2025qwen3technicalreport}, and Llama-3.1-8B-Instruct~\citep{llama3.1}. Each trained model is named \judgetable-\{base\}, e.g., \judgetable-Qwen3-4B. We compare against untrained baselines and strong LLM baselines (Table~\ref{tab:exp_main}). See Appendix~\ref{sec:app_training} for full model details.

\paragraph{Training.}
We use Group Relative Policy Optimization (GRPO)~\citep{shao2024deepseekmathpushinglimitsmathematical} with preference prediction correctness as the verifiable reward. The model generates a reasoning trace followed by a prediction (A or B), and receives reward 1 if correct, 0 otherwise. See Appendix~\ref{sec:app_training} for training configurations and computational resources.

\paragraph{Evaluation.}
To mitigate position bias, we evaluate each pair twice by swapping paper order (A$\leftrightarrow$B) and score a prediction as correct only if consistent across both orderings~\citep{zheng2023judgingllmasajudgemtbenchchatbot}. See Appendix~\ref{sec:app_evaluation} for details.

\begin{table*}[!t]
  \centering
  \small
  \setlength{\tabcolsep}{4pt}
  \renewcommand{\arraystretch}{1.2}
  \begin{tabular}{l*{5}{>{\centering\arraybackslash}p{2.0cm}}}
    \toprule
    \textbf{Model} & \textbf{CS} & \textbf{Math} & \textbf{Physics} & \textbf{Others} & \textbf{Overall} \\
    \midrule
    Qwen3-4B-Instruct & 69.7 & 69.2 & 66.0 & 53.0 & 64.7 \\
    \rowcolor{gray!10} \judgetable-Qwen3-4B & 84.0 \gain{14.3} & 89.4 \gain{20.2} & 79.0 \gain{13.0} & 74.5 \gain{21.5} & 80.9 \gain{16.2} \\
    Qwen3-30B-A3B-Instruct & 74.0 & 77.9 & 72.7 & 63.5 & 71.7 \\
    \rowcolor{gray!10} \judgetable-Qwen3-30B & 85.7 \gain{11.7} & 86.5 \gain{8.7} & 83.3 \gain{10.7} & 77.0 \gain{13.5} & 83.1 \gain{11.4} \\
    Qwen2.5-1.5B-Instruct & 30.3 & 17.3 & 20.7 & 20.0 & 23.3 \\
    \rowcolor{gray!10} \judgetable-Qwen2.5-1.5B & 63.3 \gain{33.0} & 73.1 \gain{55.8} & 68.7 \gain{48.0} & 63.5 \gain{43.5} & 66.3 \gain{42.9} \\
    Qwen2.5-3B-Instruct & 24.0 & 41.3 & 25.0 & 18.5 & 25.1 \\
    \rowcolor{gray!10} \judgetable-Qwen2.5-3B & 64.7 \gain{40.7} & 69.2 \gain{27.9} & 67.7 \gain{42.7} & 63.5 \gain{45.0} & 65.9 \gain{40.8} \\
    Qwen2.5-7B-Instruct & 45.7 & 54.8 & 54.7 & 34.0 & 47.1 \\
    \rowcolor{gray!10} \judgetable-Qwen2.5-7B & 65.7 \gain{20.0} & 76.9 \gain{22.1} & 70.7 \gain{16.0} & 66.5 \gain{32.5} & 68.8 \gain{21.7} \\
    Qwen2.5-14B-Instruct & 54.7 & 66.3 & 61.0 & 45.5 & 56.1 \\
    \rowcolor{gray!10} \judgetable-Qwen2.5-14B & 73.0 \gain{18.3} & 73.1 \gain{6.7} & 74.3 \gain{13.3} & 68.5 \gain{23.0} & 72.5 \gain{16.4} \\
    Qwen2.5-32B-Instruct & 56.0 & 67.3 & 66.0 & 50.0 & 59.3 \\
    \rowcolor{gray!10} \judgetable-Qwen2.5-32B & 70.7 \gain{14.7} & 73.1 \gain{5.8} & 77.7 \gain{11.7} & 73.5 \gain{23.5} & 73.9 \gain{14.6} \\
    Llama3.1-8B-Instruct & 41.3 & 56.7 & 51.0 & 37.5 & 45.5 \\
    \rowcolor{gray!10} \judgetable-Llama3.1-8B & 59.0 \gain{17.7} & 67.3 \gain{10.6} & 62.7 \gain{11.7} & 48.0 \gain{10.5} & 58.7 \gain{13.3} \\
    \bottomrule
  \end{tabular}
  \caption{Temporal OOD results on 904 pairs from papers published in 2025, after the training period. We report pairwise accuracy (\%) with position-swap consistency. Scores and gains are rounded to one decimal place; gains are computed from unrounded accuracies.}
  \label{tab:exp_ood_year}
\end{table*}

\subsection{Scaling Trends}

\label{sec:exp_main}

\judge learns scientific judgement effectively across all model scales and series, revealing scaling behavior with both data amount and model size (Figure~\ref{fig:scaling_curve}, Table~\ref{tab:exp_main}).

\paragraph{Data scaling leads to better performance.}
Scientific judgement performance improves steadily with more training data. The learning curves indicate an approximately log-linear relationship between data scale and performance.
From the untrained base model to the final checkpoint used for manuscript-facing comparisons, the overall score rises from 58.1 to 77.3 for Qwen3-4B and from 69.7 to 82.7 for Qwen3-30B-A3B, with gains observed in all fields.

\paragraph{Model size scaling leads to better performance.}
Scientific judgement performance improves consistently with model size. In the Qwen2.5 family, average accuracy after \judgetable training increases from 61.9 (1.5B) to 72.8 (3B), 76.6 (7B), 79.9 (14B), and 81.5 (32B).
A similar trend holds for Qwen3, where \judgetable-Qwen3-30B outperforms \judgetable-Qwen3-4B (82.7 vs.\ 77.3 average accuracy). In terms of average accuracy, \judgetable-Qwen3-30B also surpasses all listed strong LLM baselines, including GPT-5.4 Thinking.

\begin{tcolorbox}[colback=blue!8, colframe=blue!45, boxrule=0.8pt, arc=2mm,
  left=5pt, right=5pt, top=5pt, bottom=5pt]
\textbf{Takeaway 1} \quad  Learning scientific judgement
is scalable. We observe an approximately log-linear improvement in test-set performance as the amount of training data increases. Performance also improves with model size, and the final \judgetable-Qwen3-30B checkpoint surpasses all listed strong LLM baselines in average accuracy. Together, these results show that preference modeling through reinforcement learning scales with both data and model size.
\end{tcolorbox}

\begin{table*}[!t]
  \centering
  \small
  \setlength{\tabcolsep}{4pt}
  \renewcommand{\arraystretch}{1.2}
  \begin{tabular}{l*{5}{>{\centering\arraybackslash}p{2.0cm}}}
    \toprule
    & \multicolumn{1}{c}{\textit{In-Domain}} & \multicolumn{3}{c}{\textit{Out-of-Domain}} & \\
    \cmidrule(lr){2-2} \cmidrule(lr){3-5}
    \textbf{Model} & \textbf{CS} & \textbf{Math} & \textbf{Physics} & \textbf{Others} & \textbf{Overall} \\
    \midrule
    Qwen3-4B-Instruct & 58.1 & 71.1 & 51.8 & 55.8 & 58.1 \\
    \rowcolor{gray!10} \judgetable-Qwen3-4B & 75.4 \gain{17.3} & 77.9 \gain{6.9} & 60.6 \gain{8.8} & 65.9 \gain{10.1} & 68.9 \gain{10.8} \\
    Qwen3-30B-A3B-Instruct & 73.4 & 79.9 & 64.4 & 63.9 & 69.7 \\
    \rowcolor{gray!10} \judgetable-Qwen3-30B & 82.7 \gain{9.3} & 83.3 \gain{3.4} & 70.9 \gain{6.5} & 71.2 \gain{7.2} & 76.4 \gain{6.7} \\
    Qwen2.5-1.5B-Instruct & 4.0 & 6.4 & 5.3 & 5.8 & 5.3 \\
    \rowcolor{gray!10} \judgetable-Qwen2.5-1.5B & 55.7 \gain{51.6} & 63.2 \gain{56.9} & 52.4 \gain{47.1} & 51.4 \gain{45.7} & 55.2 \gain{49.9} \\
    Qwen2.5-3B-Instruct & 22.2 & 36.3 & 27.1 & 28.4 & 28.0 \\
    \rowcolor{gray!10} \judgetable-Qwen2.5-3B & 72.6 \gain{50.4} & 75.0 \gain{38.7} & 63.5 \gain{36.5} & 63.9 \gain{35.6} & 68.2 \gain{40.2} \\
    Qwen2.5-7B-Instruct & 48.0 & 53.9 & 48.8 & 51.4 & 50.2 \\
    \rowcolor{gray!10} \judgetable-Qwen2.5-7B & 74.2 \gain{26.2} & 80.9 \gain{27.0} & 66.5 \gain{17.6} & 70.7 \gain{19.2} & 72.2 \gain{22.0} \\
    Qwen2.5-14B-Instruct & 53.2 & 70.1 & 56.2 & 55.8 & 58.2 \\
    \rowcolor{gray!10} \judgetable-Qwen2.5-14B & 83.1 \gain{29.8} & 83.8 \gain{13.7} & 76.5 \gain{20.3} & 76.4 \gain{20.7} & 79.6 \gain{21.4} \\
    Qwen2.5-32B-Instruct & 62.5 & 73.5 & 60.3 & 60.1 & 63.5 \\
    \rowcolor{gray!10} \judgetable-Qwen2.5-32B & 83.5 \gain{21.0} & 87.3 \gain{13.7} & 79.1 \gain{18.8} & 78.9 \gain{18.7} & 81.8 \gain{18.3} \\
    Llama3.1-8B-Instruct & 47.6 & 53.9 & 45.3 & 38.9 & 46.3 \\
    \rowcolor{gray!10} \judgetable-Llama3.1-8B & 56.1 \gain{8.5} & 57.4 \gain{3.4} & 54.7 \gain{9.4} & 54.3 \gain{15.4} & 55.5 \gain{9.2} \\
    \bottomrule
  \end{tabular}
  \caption{Field OOD results. The trained rows use independent CS-only training runs and are evaluated on all fields. CS is in-domain while Math, Physics, and Others are out-of-domain. Scores and gains are rounded to one decimal place; gains are computed from unrounded accuracies.}
  \label{tab:exp_ood_domain}
\end{table*}

\subsection{Generalization and Controlled Comparisons}

\label{sec:exp_ood}

We now test whether learned scientific judgement generalizes beyond the training distribution along three axes: time, field, and target metric, and whether its gains persist under stricter controlled comparisons.

\paragraph{Temporal generalization to future preferences.}
Training with RLCF substantially improves prediction of future paper preferences. On papers published in 2025, Qwen3-4B improves from 64.7 to 80.9 average accuracy, while Qwen3-30B-A3B improves from 71.7 to 83.1 (Table~\ref{tab:exp_ood_year}). Gains are consistent across most base models and fields and persist when evaluation is restricted to papers published after the corresponding base models were released (Appendix~\ref{sec:app_release_time}). These post-release gains cannot be attributed to direct pretraining exposure to the evaluated papers and suggest that RLCF learns stable signals of community values that generalize beyond the training period.

\paragraph{Generalization to unseen fields.}
To evaluate cross-field generalization, we train separate \judge variants exclusively on CS papers. Both Qwen3 variants improve over their base models in every unseen field, with overall average accuracy rising from 58.1 to 68.9 for Qwen3-4B and from 69.7 to 76.4 for Qwen3-30B (Table~\ref{tab:exp_ood_domain}). This cross-field transfer is notable because different disciplines vary substantially in knowledge, style, and data distribution, yet still exhibit shared patterns of scientific value that can be learned and transferred. These results suggest that RLCF helps models acquire generalizable scientific judgement rather than merely fitting field-specific signals.

\paragraph{Generalization across community metrics.}
\judge also improves when the target metric changes from citations to other community metrics. On ICLR review-score pairs, Qwen3-4B improves by 14.1 percentage points to 79.4\% accuracy, while Qwen3-30B improves by 11.8 points to 88.5\%. On Altmetric attention-score pairs, Qwen3-4B improves by 12.2 percentage points to 82.6\% accuracy, while Qwen3-30B improves by 8.2 points to 86.3\% (Table~\ref{tab:exp_ood_iclr}). These results indicate that citation-trained judgement transfers to related but distinct community metrics. Appendix~\ref{sec:app_metric_tests} defines both metrics and details the test-set construction.

\begin{table*}[!t]
  \centering
  \footnotesize
  \setlength{\tabcolsep}{5pt}
  \renewcommand{\arraystretch}{1.2}

  \begin{subtable}[t]{0.495\textwidth}
    \centering
    \begin{tabular}{lcc}
      \toprule
      \textbf{Model} & & \textbf{Acc. (\%)} \\
      \midrule
      Qwen3-4B-Instruct & & 65.3 \\
      \rowcolor{gray!10} \judgetable-Qwen3-4B & & 79.4 \gain{14.1} \\
      Qwen3-30B-A3B-Instruct & & 76.8 \\
      \rowcolor{gray!10} \judgetable-Qwen3-30B & & 88.5 \gain{11.8} \\
      Qwen2.5-1.5B-Instruct & & 1.6 \\
      \rowcolor{gray!10} \judgetable-Qwen2.5-1.5B & & 73.7 \gain{72.0} \\
      Qwen2.5-3B-Instruct & & 15.2 \\
      \rowcolor{gray!10} \judgetable-Qwen2.5-3B & & 78.4 \gain{63.2} \\
      Qwen2.5-7B-Instruct & & 46.6 \\
      \rowcolor{gray!10} \judgetable-Qwen2.5-7B & & 78.2 \gain{31.6} \\
      Qwen2.5-14B-Instruct & & 42.6 \\
      \rowcolor{gray!10} \judgetable-Qwen2.5-14B & & 85.4 \gain{42.9} \\
      Qwen2.5-32B-Instruct & & 58.4 \\
      \rowcolor{gray!10} \judgetable-Qwen2.5-32B & & 84.8 \gain{26.4} \\
      Llama3.1-8B-Instruct & & 46.6 \\
      \rowcolor{gray!10} \judgetable-Llama3.1-8B & & 68.9 \gain{22.3} \\
      \bottomrule
    \end{tabular}
    \caption{ICLR Metric OOD}
    \label{tab:exp_ood_iclr_panel}
  \end{subtable}
  \hfill
  \begin{subtable}[t]{0.495\textwidth}
    \centering
    \begin{tabular}{lcc}
      \toprule
      \textbf{Model} & & \textbf{Acc. (\%)} \\
      \midrule
      Qwen3-4B-Instruct & & 70.5 \\
      \rowcolor{gray!10} \judgetable-Qwen3-4B & & 82.6 \gain{12.2} \\
      Qwen3-30B-A3B-Instruct & & 78.1 \\
      \rowcolor{gray!10} \judgetable-Qwen3-30B & & 86.3 \gain{8.2} \\
      Qwen2.5-1.5B-Instruct & & 8.2 \\
      \rowcolor{gray!10} \judgetable-Qwen2.5-1.5B & & 63.1 \gain{54.9} \\
      Qwen2.5-3B-Instruct & & 28.0 \\
      \rowcolor{gray!10} \judgetable-Qwen2.5-3B & & 76.3 \gain{48.2} \\
      Qwen2.5-7B-Instruct & & 59.3 \\
      \rowcolor{gray!10} \judgetable-Qwen2.5-7B & & 79.0 \gain{19.7} \\
      Qwen2.5-14B-Instruct & & 67.1 \\
      \rowcolor{gray!10} \judgetable-Qwen2.5-14B & & 84.1 \gain{17.0} \\
      Qwen2.5-32B-Instruct & & 71.0 \\
      \rowcolor{gray!10} \judgetable-Qwen2.5-32B & & 83.0 \gain{12.0} \\
      Llama3.1-8B-Instruct & & 49.1 \\
      \rowcolor{gray!10} \judgetable-Llama3.1-8B & & 64.9 \gain{15.9} \\
      \bottomrule
    \end{tabular}
    \caption{Altmetric Metric OOD}
    \label{tab:exp_ood_altmetric_panel}
  \end{subtable}

  \vspace{2pt}
  \caption{Metric OOD results using ICLR peer-review scores (left) and Altmetric attention scores (right) as preference signals instead of citations. Scores and gains are rounded to one decimal place; gains are computed from unrounded accuracies.}
  \label{tab:exp_ood_iclr}
\end{table*}

\paragraph{Judgement gains persist under author, institution, and topic controls.}
On the 541-pair author and institution control set and the 245-pair embedding-based topic-control set, both Qwen3 model sizes improve consistently, with gains ranging from 8.3 to 11.8 percentage points (Fig.~\ref{fig:controlled_comparisons}). These results reduce simple explanations based solely on author, institutional, temporal, or broad topical cues. Appendix~\ref{sec:app_controlled_tests} describes the external resources and matching procedures used for both controls.

\begin{figure*}[!t]
  \centering
  \includegraphics[width=\textwidth]{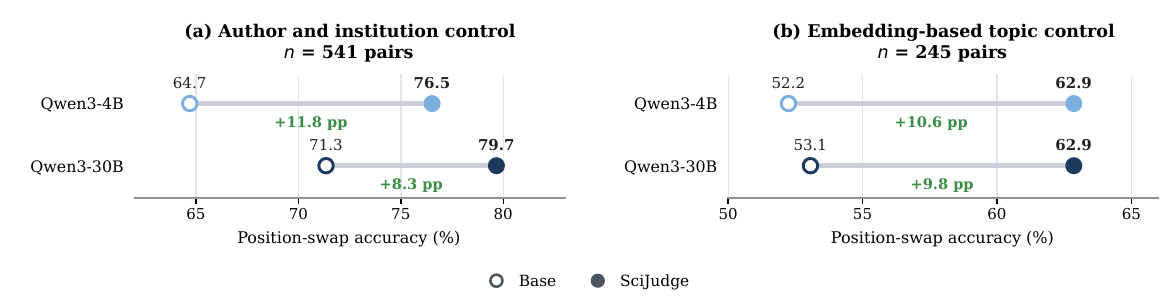}
  \caption{Controlled comparisons under matched author, institution, and topic context. (a) Author and institution control (541 pairs). (b) Embedding-based topic control (245 pairs). Hollow and filled markers denote the base and \judge variants, respectively, and lines connect variants of the same model. Accuracy requires correct predictions under both A/B orderings; green labels show gains in percentage points. Panels use dataset-specific accuracy scales.}
  \label{fig:controlled_comparisons}
\end{figure*}

\noindent We additionally verify that training preserves general-purpose capabilities (Appendix~\ref{sec:app_general}). Given this generalizability, we next ask whether \judge can serve as a reward signal for improving scientific ideation.

\begin{tcolorbox}[colback=blue!8, colframe=blue!45, boxrule=0.8pt, arc=2mm,
  left=5pt, right=5pt, top=5pt, bottom=5pt]
\textbf{Takeaway 2} \quad Learned scientific judgement generalizes across time to future papers, across fields beyond the CS-only training distribution, and across community metrics to peer-review scores and Altmetric attention scores. Its gains also persist under controlled comparisons matching author and institutional context, publication time, and topic similarity. Together, these results suggest that \judge captures transferable patterns from community feedback, providing evidence that AI can learn a broadly generalizable component of scientific taste.
\end{tcolorbox}

\section{AI Can Learn Ideation with High Potential Impact}
\label{sec:exp_generation}

In this section, we focus on training \thinker using Comparison-Based GRPO (\S~\ref{sec:method_thinker}) with \judge as the reward model (\S~\ref{sec:exp_judge}).

\subsection{Experimental Setup}
\label{sec:exp_gen_setup}

\paragraph{Data.}
We use high-citation papers from 2025 as seed papers. The training set consists of 4,000 papers published between January and July. For evaluation, we use 200 papers from the same period as an in-domain test set and 200 papers from August-December as an out-of-domain test set.

\paragraph{Models.}
We train \thinker on two policy models: Qwen3-30B-A3B-Thinking-2507 and Qwen3-4B-Thinking-2507, both using \judgetable-Qwen3-4B as the reward model. We refer to the two trained policies as \thinkertable-30B and \thinkertable-4B. To explore the gains from preference learning, we also train a version of each policy using the base model of \judgetable-Qwen3-4B (Qwen3-4B-Instruct) as the reward model.

\paragraph{Evaluation.}
We evaluate \thinker by its win rate against the base policy. For each seed paper, both models propose a research idea, and we use three strong LLMs (GPT-5.2-high, GLM-5 and Gemini 3 Pro) to judge which idea has higher potential impact via majority vote, detailed in Appendix~\ref{sec:app_thinker_eval}. In the same way, we also evaluate \thinkertable-30B's win rates against the three strong LLM baselines.

\begin{figure}[t]
    \centering
    \begin{subfigure}[t]{0.49\textwidth}
        \centering
        \includegraphics[width=\linewidth]{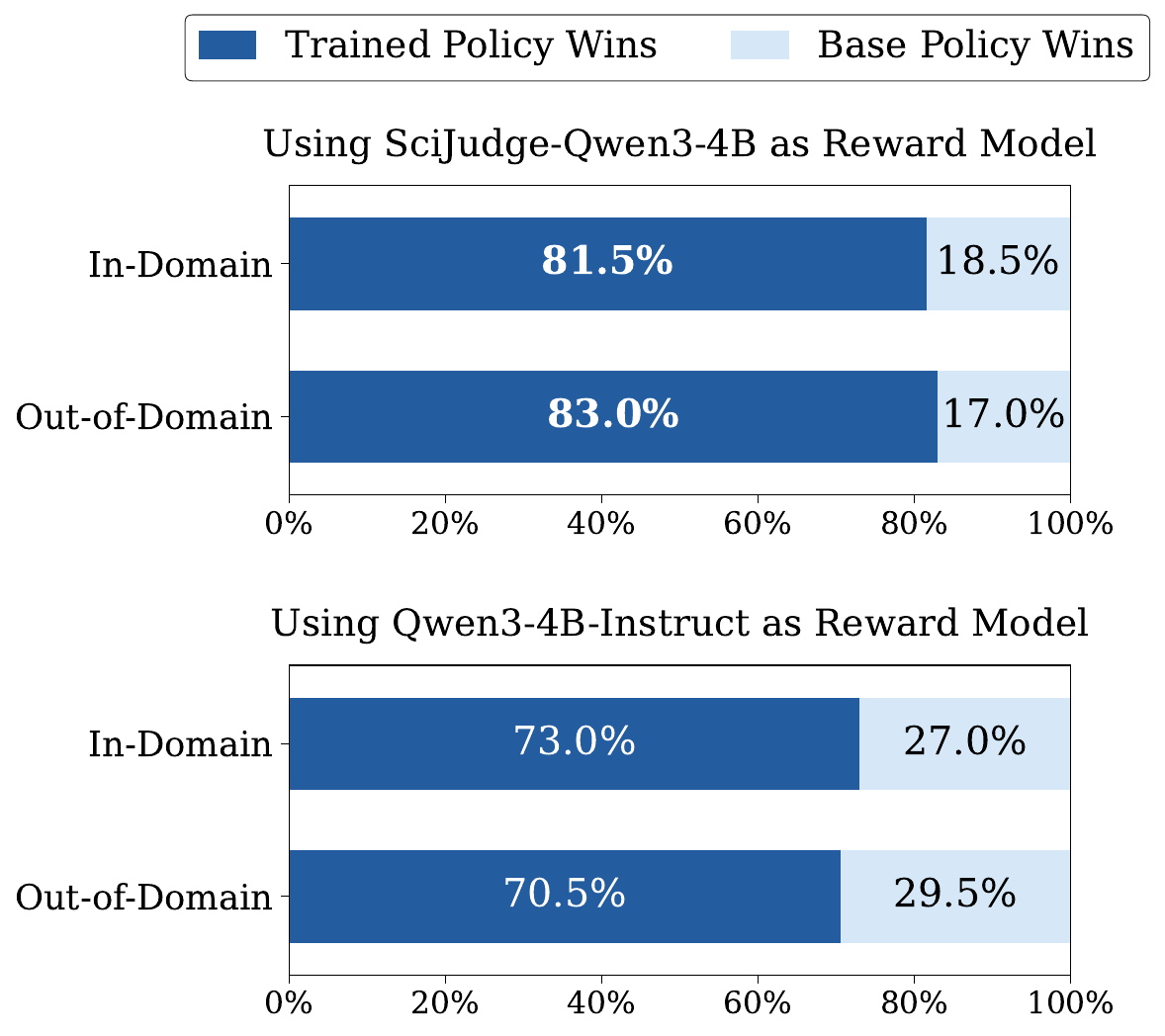}
        \caption{Base Policy: Qwen3-30B-A3B-Thinking-2507}
        \label{fig:thinker_30b}
    \end{subfigure}
    \hfill
    \begin{subfigure}[t]{0.49\textwidth}
        \centering
        \includegraphics[width=\linewidth]{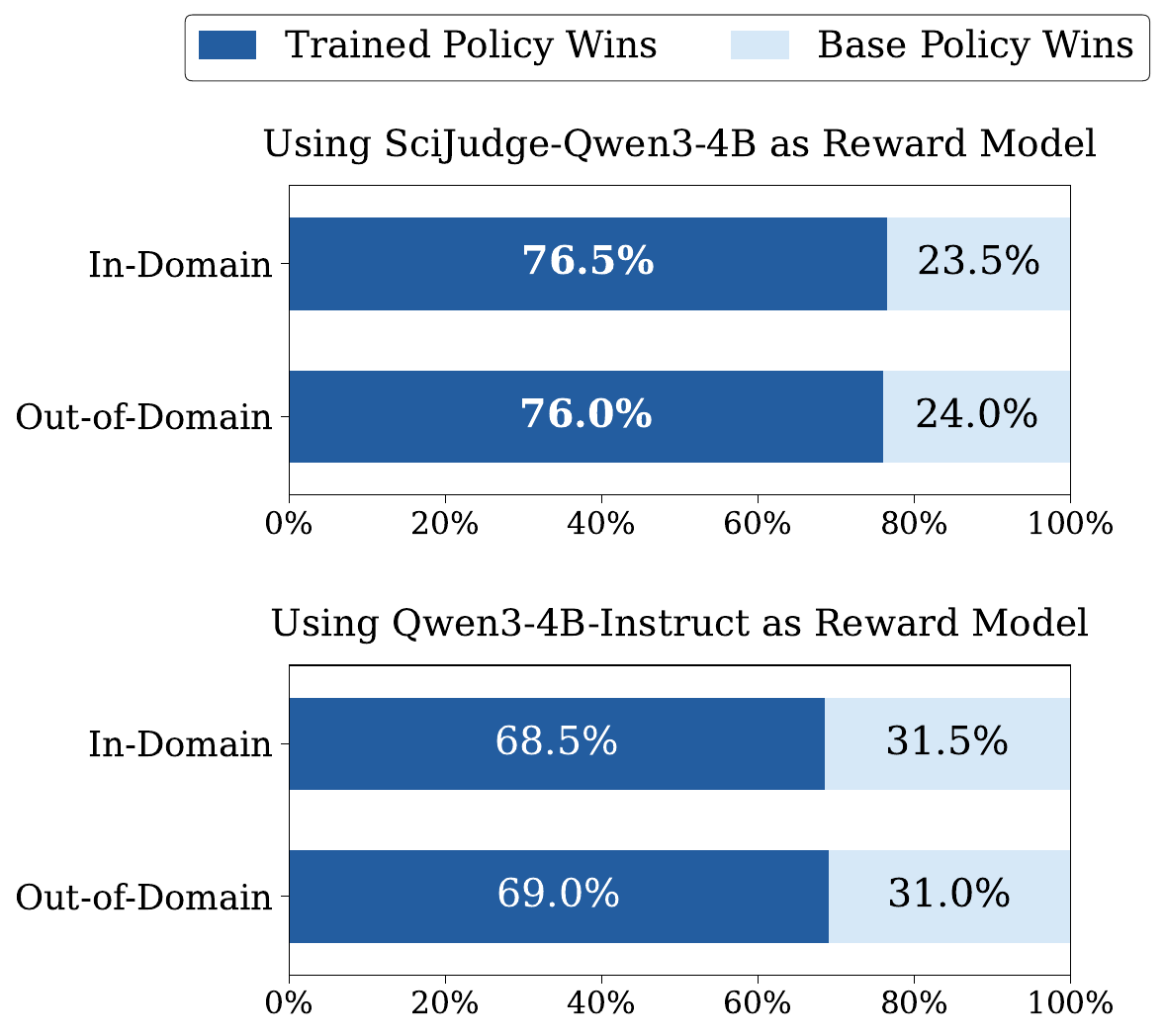}
        \caption{Base Policy: Qwen3-4B-Thinking-2507}
        \label{fig:thinker_4b}
    \end{subfigure}
    
    \caption{\thinker's performance under different base policy models and reward models. The top row uses \judgetable-Qwen3-4B as the reward model, while the bottom row uses the baseline reward model, Qwen3-4B-Instruct.}
    \label{fig:exp_human_eval}
\end{figure}

\subsection{Results}

\paragraph{Substantial improvements in scientific ideation with high potential impact.} As shown in Figure~\ref{fig:exp_human_eval}, after training with the \judgetable-Qwen3-4B reward model, \thinker remarkably outperforms the base policy at both model sizes, achieving in-domain win rates of 81.5\% and 76.5\% for the 30B and 4B models, respectively. The gains also transfer to seed papers published after the training period (30B: 83.0\%, 4B: 76.0\%), indicating that the ideation improvements persist on this temporal holdout.

\paragraph{\judge is a more effective generative reward model than the baseline.} For both model sizes, the policy trained with \judgetable-Qwen3-4B significantly outperforms the one trained with Qwen3-4B-Instruct (Figure~\ref{fig:exp_human_eval}) on in-domain (e.g., 30B: 81.5\% vs.\ 73.0\%) and out-of-domain (e.g., 30B: 83.0\% vs.\ 70.5\%) test sets.

\paragraph{\thinker is comparable to strong LLM baselines in ideation.} \thinkertable-30B surpasses GPT-5.2 and GLM-5 after training (Table~\ref{tab:exp_winrate_sota}), achieving an average win-rate of 54.2\% across the three strong LLM baselines.

\begin{table*}[!t]
  \centering
  \footnotesize
  \setlength{\tabcolsep}{1.2pt}
  \renewcommand{\arraystretch}{1.2}
  
  \begin{subtable}[t]{0.495\textwidth}
    \centering
    \begin{tabular}{lcccc}
      \toprule
      \textbf{Model} & \textbf{GPT-5.2} & \textbf{GLM-5} & \textbf{Gemini 3 Pro} & \textbf{Avg.} \\
      \midrule
      Qwen3-30B & 37.5 & 33.0 & 20.5 & 30.3 \\
      \rowcolor{gray!10} \thinkertable-30B & 61.0\gain{23.5} & 58.5\gain{25.5} & 43.0\gain{22.5} & 54.2\gain{23.9} \\
      \bottomrule
    \end{tabular}
    \caption{In-Domain Win Rates (\%)}
    \label{tab:winrate_indomain}
  \end{subtable}
  \hfill
  \begin{subtable}[t]{0.495\textwidth}
    \centering
    \begin{tabular}{lcccc}
      \toprule
      \textbf{Model} & \textbf{GPT-5.2} & \textbf{GLM-5} & \textbf{Gemini 3 Pro} & \textbf{Avg.} \\
      \midrule
      Qwen3-30B & 36.0 & 29.5 & 18.0 & 27.8 \\
      \rowcolor{gray!10} \thinkertable-30B & 59.0\gain{23.0} & 61.0\gain{31.5} & 42.5\gain{24.5} & 54.2\gain{26.4} \\
      \bottomrule
    \end{tabular}
    \caption{Out-of-Domain Win Rates (\%)}
    \label{tab:winrate_ood}
  \end{subtable}
  
  \vspace{2pt}
  \caption{Win rates of \thinkertable-30B and its base policy (Qwen3-30B-A3B-Thinking) against three strong LLM baselines.}
  \label{tab:exp_winrate_sota}
\end{table*}

\begin{tcolorbox}[colback=blue!8, colframe=blue!45, boxrule=0.8pt, arc=2mm,
  left=5pt, right=5pt, top=5pt, bottom=5pt]
\textbf{Takeaway 3}\quad
AI can learn ideation with high potential impact. Through RL training, \thinker proposes scientific ideas with higher potential impact, and the gains transfer to seed papers published after the training period. Furthermore, \judge serves as an effective reward model for training \thinker in these experiments. Therefore, RLCF successfully enhances both scientific judgement and ideation capabilities of models, offering an effective and promising pathway for AI to learn scientific taste.
\end{tcolorbox}

\section{Conclusion}
\label{sec:conclusion}

In this work, we use the term \textit{scientific taste} to refer to the ability to judge and propose research ideas with potential for long-term scientific impact. To learn this capability, we propose RLCF, which uses citation signals as community feedback for preference modeling and alignment, yielding \judge for scientific judgement and \thinker for scientific ideation.
Experiments show that \judge scales with data and model size, generalises to future-year papers, other community metrics, and unseen fields, and retains its gains under controlled comparisons matching author and institutional context, publication time, and topic similarity. Using \judge as a generative reward model, \thinker proposes research ideas judged to have higher potential impact than those proposed by strong LLM baselines.
Overall, our results suggest that AI can learn scientific taste from large-scale community feedback, marking an important step towards AI systems that could help accelerate scientific discovery.

\section*{Limitations and Future Work}

Our work has several limitations. First, scientific taste may involve more than scientific judgement and ideation with high potential impact. Future work could explore broader formulations of scientific taste such as assessing the feasibility of experimental designs and better recognizing distinctive and diverse research ideas.

Second, citations are an imperfect form of community feedback. Some high-potential papers may receive few citations initially but become highly influential later. Modeling citation dynamics may help capture such delayed-impact patterns. In addition, our field categorization remains limited in granularity, and more fine-grained field clustering may improve the quality of preference pairs. The controlled evaluations reduce simple author-, institution-, and topic-level explanations. However, they cannot remove all biases in citation-based community feedback.

Third, our ideation evaluation mainly relies on strong LLM evaluators. Since the proposed ideas are not experimentally validated, the evaluation may not fully reflect their potential impact. Future work can implement a subset of these ideas.

Finally, \judge is trained mainly on titles and abstracts. Incorporating richer paper context, such as related work sections, may improve scientific judgement.

\section*{Ethical Considerations}

This work uses publicly available paper metadata and does not involve private or sensitive information. Citation-based signals may encode biases across fields or topics. To mitigate this, we construct preference pairs within the same subcategory and similar publication time windows. We also evaluate controlled comparisons matching author and institutional context, publication time, and topic similarity. Such systems should assist, rather than replace, human judgement in scientific evaluation. In addition, scientific ideation models may be misused to generate low-quality ideas at scale or facilitate academic misconduct. We therefore encourage responsible use under human oversight.

\bibliographystyle{unsrtnat}
\bibliography{custom}

\newpage
\beginappendix

\startcontents[app]
\begingroup
  \renewcommand{\contentsname}{Appendix Contents}
  \section*{\contentsname}
  \printcontents[app]{}{1}{}
\endgroup
\newpage

\section{Dataset Construction Details}
\label{sec:app_dataset}

This appendix provides detailed procedures for constructing \bench, including training data and test sets.

\subsection{Data Statistics}

Table~\ref{tab:app_data_stats} summarizes the dataset statistics across different fields.

\begin{table}[h]
  \centering
  \small
  \begin{tabular}{lrr}
    \toprule
    \textbf{Field} & \textbf{\# Paper records} & \textbf{\# Pairs} \\
    \midrule
    Computer Science & 151,060 & 75,530 \\
    Mathematics & 64,364 & 32,182 \\
    Physics & 1,196,852 & 598,426 \\
    Others & 28,406 & 14,203 \\
    \midrule
    \textbf{Total} & \textbf{1,440,682} & \textbf{720,341} \\
    \bottomrule
  \end{tabular}
  \caption{Training dataset statistics by field. Paper records are pair-level counts: each pair contributes two records, and the same paper may recur across pairs; these are not counts of unique papers.}
  \label{tab:app_data_stats}
\end{table}

\paragraph{Field-to-Subcategory Mapping.}
We group papers by their primary arXiv category. The four top-level fields correspond to the following subcategories. In particular, \textit{Others} is an explicit aggregation of non-CS/Math/Physics areas, specifically covering Economics, Electrical Engineering and Systems Science, Quantitative Biology, Quantitative Finance, and Statistics, rather than a residual bucket.

\begingroup
\small
\begin{itemize}[leftmargin=*,nosep]
    \item \textbf{Computer Science:} \texttt{cs.CV}, \texttt{cs.LG}, \texttt{cs.CL}, \texttt{cs.RO}, \texttt{cs.AI}, \texttt{cs.CR}, \texttt{cs.SE}, \texttt{cs.HC}, \texttt{cs.IR}, \texttt{cs.CY}, \texttt{cs.IT}, \texttt{cs.DC}, \texttt{cs.SD}, \texttt{cs.NI}, \texttt{cs.DS}, \texttt{cs.AR}, \texttt{cs.GT}, \texttt{cs.SI}, \texttt{cs.LO}, \texttt{cs.GR}, \texttt{cs.CE}, \texttt{cs.DB}, \texttt{cs.MA}, \texttt{cs.NE}, \texttt{cs.PL}, \texttt{cs.CC}, \texttt{cs.ET}, \texttt{cs.MM}, \texttt{cs.DL}, \texttt{cs.CG}, \texttt{cs.FL}, \texttt{cs.DM}, \texttt{cs.PF}, \texttt{cs.SC}, \texttt{cs.OS}, \texttt{cs.MS}, \texttt{cs.OH}, \texttt{cs.SY}, \texttt{cs.GL}.
    \item \textbf{Physics:} \texttt{quant-ph}, \texttt{cond-mat.mtrl-sci}, \texttt{hep-ph}, \texttt{astro-ph.GA}, \texttt{gr-qc}, \texttt{astro-ph.HE}, \texttt{physics.optics}, \\\texttt{hep-th}, \texttt{cond-mat.mes-hall}, \texttt{astro-ph.SR}, \texttt{astro-ph.CO}, \texttt{physics.flu-dyn}, \texttt{astro-ph.EP}, \\\texttt{cond-mat.str-el}, \texttt{cond-mat.soft}, \texttt{astro-ph.IM}, \texttt{cond-mat.stat-mech}, \texttt{physics.chem-ph}, \\\texttt{physics.ins-det}, \texttt{nucl-th}, \texttt{physics.plasm-ph}, \texttt{cond-mat.supr-con}, \texttt{math-ph}, \texttt{physics.soc-ph}, \texttt{hep-ex}, \texttt{physics.app-ph}, \texttt{physics.comp-ph}, \texttt{cond-mat.quant-gas}, \texttt{physics.med-ph}, \texttt{physics.atom-ph}, \texttt{hep-lat}, \texttt{physics.ao-ph}, \texttt{physics.geo-ph}, \texttt{physics.acc-ph}, \texttt{physics.bio-ph}, \texttt{cond-mat.dis-nn}, \texttt{nucl-ex}, \texttt{nlin.CD}, \texttt{physics.ed-ph}, \texttt{physics.gen-ph}, \texttt{physics.class-ph}, \texttt{nlin.SI}, \texttt{nlin.PS}, \texttt{physics.hist-ph}, \texttt{nlin.AO}, \\\texttt{physics.space-ph}, \texttt{cond-mat.other}, \texttt{physics.data-an}, \texttt{physics.atm-clus}, \texttt{physics.pop-ph}, \texttt{nlin.CG}.
    \item \textbf{Mathematics:} \texttt{math.AP}, \texttt{math.OC}, \texttt{math.CO}, \texttt{math.NA}, \texttt{math.PR}, \texttt{math.NT}, \texttt{math.AG}, \texttt{math.DG}, \texttt{math.DS}, \texttt{math.FA}, \texttt{math.ST}, \texttt{math.RT}, \texttt{math.GR}, \texttt{math.GT}, \texttt{math.RA}, \texttt{math.CA}, \texttt{math.LO}, \texttt{math.CV}, \texttt{math.AC}, \texttt{math.AT}, \texttt{math.MG}, \texttt{math.OA}, \texttt{math.QA}, \texttt{math.CT}, \texttt{math.GM}, \texttt{math.SP}, \texttt{math.SG}, \texttt{math.HO}, \texttt{math.GN}, \texttt{math.KT}.
    \item \textbf{Others:} \texttt{stat.ME}, \texttt{stat.ML}, \texttt{stat.AP}, \texttt{econ.GN}, \texttt{q-bio.NC}, \texttt{q-bio.QM}, \texttt{q-bio.PE}, \texttt{econ.TH}, \texttt{econ.EM}, \texttt{stat.CO}, \texttt{q-bio.BM}, \texttt{q-bio.GN}, \texttt{q-fin.MF}, \texttt{q-fin.CP}, \texttt{q-fin.ST}, \texttt{q-fin.TR}, \texttt{q-fin.RM}, \texttt{q-fin.PM}, \texttt{q-bio.OT}, \texttt{q-bio.MN}, \texttt{q-bio.TO}, \texttt{q-fin.GN}, \texttt{q-fin.PR}, \texttt{stat.OT}, \texttt{q-bio.CB}, \texttt{q-bio.SC}, \texttt{eess.SP}, \texttt{eess.SY}, \texttt{eess.IV}, \texttt{eess.AS}.
\end{itemize}
\endgroup

\subsection{Training Data Construction}

\paragraph{Paper Collection.}
We collect arXiv papers published through December 7, 2025. From the full arXiv metadata archive of 2.9M papers, we obtain citation counts for 2.3M papers and select 2.1M papers published through 2024 as the training paper pool. Each paper record includes title, abstract, publication date, subcategory, and citation count. This pool covers Computer Science, Mathematics, Physics, and other scientific fields.

\paragraph{Pair Generation.}
We generate preference pairs by matching papers within the same subcategory and similar publication time windows. For each pair, the paper with higher citations is labeled as preferred. Let $c_{\mathrm{hi}}$ and $c_{\mathrm{lo}}$ denote the higher and lower citation counts in a candidate pair. We keep a pair only if
\[
c_{\mathrm{hi}} - c_{\mathrm{lo}} \ge 8,
\qquad
\frac{c_{\mathrm{hi}} - c_{\mathrm{lo}}}{c_{\mathrm{hi}}} \ge 0.3.
\]
Equivalently, the relative citation difference is computed with respect to the higher-citation paper. These criteria correspond to:
\begin{itemize}[nosep]
    \item Absolute citation difference $\geq$ 8
    \item Relative citation difference $\geq$ 30\%
\end{itemize}
This results in 720,341 field- and time-matched training pairs, corresponding to 1,440,682 pair-level paper records across fields.

\subsection{Test Set Construction}

\subsubsection{Citation-Based Test Sets}
\label{sec:app_citation_tests}

\paragraph{Main Test Set.}
The main (in-domain) test set consists of 1,000 citation-preference pairs sampled from a full 8,830-pair source pool. The source pool is constructed from the same field- and time-matched citation-pairing procedure as the training data, but with stricter evaluation thresholds: absolute citation difference $\geq 32$ and relative citation difference $\geq 50\%$, i.e., $(c_{\mathrm{hi}} - c_{\mathrm{lo}})/c_{\mathrm{hi}} \ge 0.5$. All papers in the full source pool are excluded from the training set, ensuring paper-level disjointness between the training and main evaluation data. We use the sampled 1,000-pair set for routine evaluation to reduce repeated evaluation cost while retaining a representative and discriminative benchmark.

The sampler preserves the original pair records, without rewriting prompts or changing A/B labels. It enforces an exact answer balance of 500 A-labelled and 500 B-labelled pairs. It matches the source-pool year and top-level field distributions using largest-remainder allocation, and jointly matches year-field-answer targets using min-cost-flow allocation. Within each constrained cell, rows are selected by deterministic hash order. We additionally monitor the first-two-subcategory-pair and citation-gap-decile distributions as representativeness diagnostics.

\paragraph{Temporal OOD Test Set.}
To evaluate generalization to future papers, we construct a category-adaptive temporal OOD test set from arXiv papers published in 2025 using the NASA ADS citation snapshot dated 22 July 2026, ensuring complete temporal separation from the training pool (published through 2024). We map papers to Computer Science, Mathematics, Physics, or Others using the official arXiv category taxonomy. Candidate pairs must share the same natural publication month and the same ordered first two arXiv subcategories; we impose no additional within-month day-gap limit and do not reuse papers. For higher- and lower-citation papers with counts $c_{\mathrm{hi}}$ and $c_{\mathrm{lo}}$, respectively, we require
\[
\frac{c_{\mathrm{hi}}-c_{\mathrm{lo}}}{c_{\mathrm{hi}}}\geq 0.5,
\qquad
c_{\mathrm{hi}}-c_{\mathrm{lo}}\geq\delta_g.
\]
For each top-level field $g$, we enumerate the complete integer threshold grid $\delta_g\in\{12,\ldots,32\}$ with seed 42, using the same deterministic two-stage shuffling and first-valid sequential pairing logic as the main test builder. We choose the highest threshold whose complete pair pool reaches the field cap (300 pairs for Computer Science and Physics; 200 for Mathematics and Others). If no threshold reaches the cap, we choose the threshold yielding the most pairs, breaking ties toward the higher threshold. This selects $\delta_{\mathrm{CS}}=\delta_{\mathrm{Physics}}=32$ and $\delta_{\mathrm{Mathematics}}=\delta_{\mathrm{Others}}=12$. Caps are applied only after pairing through deterministic month-balanced sampling over intact pairs, followed by a fixed-seed final shuffle and per-batch A/B adjustment. The final set contains 904 non-overlapping pairs (Computer Science: 300; Mathematics: 104; Physics: 300; Others: 200), comprising 1,808 unique papers with exact A/B balance (452/452).

\subsubsection{Metric OOD Test Sets}
\label{sec:app_metric_tests}

\paragraph{ICLR Metric OOD Test Set.}
The International Conference on Learning Representations (ICLR) is a peer-reviewed machine-learning venue spanning representation learning and broader areas of machine learning. Its submissions receive numerical reviewer ratings and reviewer confidence scores. We use mean ratings to define preferences and test transfer from citation-based training to peer-review evaluation. Confidence and rating variance are used only for quality filtering.

The test set uses ICLR submissions from 2017 to 2026. We apply two quality filters within each year, removing papers whose average reviewer confidence falls in the bottom 50\% and papers whose rating variance falls in the top 50\%. We retain only the top and bottom 10\% of papers by average review rating within each year, capped at 75 papers per side. We sort the retained papers by rating, split them at the median, randomly shuffle within each half, and pair papers one-to-one across the two halves. The paper from the upper half is labelled preferred. This yields 611 test pairs with a median rating difference of 6.1 points on the 1--10 review scale.

\paragraph{Altmetric Metric OOD Test Set.}
Altmetric tracks online attention to research outputs across social media, mainstream news, policy documents, patents, and blogs. Its attention score is a weighted indicator that complements citation metrics, but it measures attention rather than scientific quality or societal impact. We use this score to test transfer from citation preferences to a distinct attention-based signal.

We construct the test set from the full 8,830-pair main source pool. Among 17,660 paper slots, 13,360 have available Altmetric attention scores; the remaining 4,300 are excluded. Within the same first-two arXiv subcategory and publication year, we sequentially pair papers without global paper reuse. We require an absolute attention-score difference of at least 96 and a relative difference of at least 50\%. The paper with the higher Altmetric attention score is labelled preferred. This yields 599 pairs with 1,198 unique arXiv IDs and near-exact answer balance (300 A-labelled and 299 B-labelled pairs).

\subsubsection{Controlled-Comparison Test Sets}
\label{sec:app_controlled_tests}

\paragraph{Author and Institution Control Test Set.}
CSRankings is a publication-based ranking of computer science departments built from faculty affiliations and publication records~\citep{csrankings2026}. We use faculty identities and affiliations only to construct matched pairs; institutional ranking positions do not define preference labels. Citation counts come from the NASA Astrophysics Data System (ADS), a bibliographic database that indexes arXiv papers and their citation links.

To test whether citation-derived preferences remain informative under matched author and institutional context, we collect CSRankings faculty arXiv papers from 2015 to 2024 and exclude papers overlapping with the main in-domain source pool. Each pair is linked to the same CSRankings faculty member and shares the same publication quarter and first-two arXiv subcategory. We require an absolute citation difference of at least 12 and a relative difference of at least 50\%, then select a maximum-cardinality non-overlapping matching. This yields 541 pairs with 1,082 unique arXiv IDs and near-exact answer balance (271 A-labelled and 270 B-labelled pairs).

\paragraph{Embedding-Based Topic Control Test Set.}
Text embeddings map paper titles and abstracts to vectors whose cosine similarity provides a continuous proxy for semantic topic similarity. We use this signal to match topics more tightly than field, subcategory, and year alone.

We construct the topic-control set from the full 8,830-pair main source pool. We represent each paper using normalized Qwen3-Embedding-8B embeddings~\citep{zhang2025qwen3embedding} over the text schema \texttt{Title: ...\textbackslash nAbstract: ...}. Candidate pairs must share the same top-level field, first-two arXiv subcategory, and publication year, and original main-source pairs are excluded. We select non-overlapping pairs with cosine similarity at least 0.70 and citation difference at least 32. This yields 245 pairs with no original-pair overlap, median embedding similarity 0.749, median citation difference 64, and near-exact answer balance (123 A-labelled and 122 B-labelled pairs).

\subsection{Field OOD Training Data}

For field OOD experiments, we construct a CS-only training set by filtering the full training data to include only Computer Science papers. This allows us to evaluate whether models trained on a single field can generalize to other fields.

\subsection{Biology Field OOD (bioRxiv)}
\label{sec:app_bio_ood}

To further probe cross-field transfer, we evaluate models trained on arXiv papers (covering CS, Math, Physics, and other fields) on bioRxiv papers from the biology field---a platform and field entirely absent from the training data. We construct 160 preference pairs from bioRxiv papers using the same citation-based pairing procedure, with thresholds of absolute citation difference $\geq$ 24 and relative difference $\geq$ 75\% (median citation difference = 134). Papers span biology subdisciplines (e.g., bioinformatics, genomics, neuroscience) and are paired within the same subdiscipline.

\paragraph{Generalization to biology preprints.}
On the 160 bioRxiv pairs, training yields gains of 3.1 and 5.6 percentage points for Qwen3-4B and Qwen3-30B, reaching 56.3\% and 70.6\% accuracy, respectively (Table~\ref{tab:exp_ood_bio}). Although these gains are smaller than those in the main generalization settings, they provide additional evidence that citation-trained scientific judgement can transfer to biology preprints.

\begin{table}[!ht]
  \centering
  \small
  \renewcommand{\arraystretch}{1.2}
  \begin{tabular}{lcc}
    \toprule
    \textbf{Model} & & \textbf{Acc.} \\
    \midrule
    Qwen3-4B-Instruct & & 53.1 \\
    \rowcolor{gray!10} \judgetable-Qwen3-4B & & 56.3 \gain{3.1} \\
    Qwen3-30B-A3B-Instruct & & 65.0 \\
    \rowcolor{gray!10} \judgetable-Qwen3-30B & & 70.6 \gain{5.6} \\
    Qwen2.5-1.5B-Instruct & & 5.0 \\
    \rowcolor{gray!10} \judgetable-Qwen2.5-1.5B & & 42.5 \gain{37.5} \\
    Qwen2.5-3B-Instruct & & 12.5 \\
    \rowcolor{gray!10} \judgetable-Qwen2.5-3B & & 55.0 \gain{42.5} \\
    Qwen2.5-7B-Instruct & & 36.3 \\
    \rowcolor{gray!10} \judgetable-Qwen2.5-7B & & 61.9 \gain{25.6} \\
    Qwen2.5-14B-Instruct & & 50.6 \\
    \rowcolor{gray!10} \judgetable-Qwen2.5-14B & & 64.4 \gain{13.8} \\
    Qwen2.5-32B-Instruct & & 46.9 \\
    \rowcolor{gray!10} \judgetable-Qwen2.5-32B & & 68.1 \gain{21.3} \\
    Llama3.1-8B-Instruct & & 28.8 \\
    \rowcolor{gray!10} \judgetable-Llama3.1-8B & & 38.8 \gain{10.0} \\
    \bottomrule
  \end{tabular}
  \caption{Biology field OOD: models trained on arXiv (CS, Math, Physics, and Others) tested on bioRxiv papers. Scores and gains are rounded to one decimal place; gains are computed from unrounded accuracies.}
  \label{tab:exp_ood_bio}
\end{table}

\section{Training and Evaluation of \judge}
\label{sec:app_training}

This appendix provides detailed training configurations for \judge.

\subsection{Base Models}

We train \judge on several open-source instruction-tuned language models. Our selection is designed to evaluate two key aspects: (1) scaling behavior across model sizes by including the Qwen2.5-Instruct series from 1.5B to 32B parameters, and (2) cross-family generalization by incorporating models from different families (Qwen2.5, Qwen3, and Llama).

\paragraph{Qwen2.5-Instruct Series~\cite{qwen2.5}.}
We use the Qwen2.5-Instruct series across multiple scales: 1.5B, 3B, 7B, 14B, and 32B parameters. These models are instruction-tuned variants of Qwen2.5, a transformer-based language model trained on large-scale multilingual data. The Qwen2.5 series demonstrates strong performance on reasoning, mathematics, and code generation tasks while maintaining efficiency through techniques like Grouped-Query Attention (GQA).

\paragraph{Qwen3-Instruct Series~\cite{yang2025qwen3technicalreport}}
We include two instruction-tuned models from the Qwen3 series: Qwen3-4B-Instruct-2507 and Qwen3-30B-A3B-Instruct-2507. These models extend our evaluation beyond the Qwen2.5 series and allow us to test whether the preference training approach transfers to a newer model family at different scales.

\paragraph{Llama-3.1-8B-Instruct~\cite{llama3.1}.}
Llama-3.1-8B-Instruct is Meta's 8-billion-parameter instruction-tuned model. It features an extended context length of 128K tokens and is trained on a diverse multilingual corpus. We include this model to evaluate cross-family generalization of our preference training approach.

\paragraph{Model Naming Convention.}
Table~\ref{tab:model_names} lists the correspondence between short names used in this paper and official base model identifiers. All \judge variants are trained from the corresponding base model using GRPO on \bench.

\begin{table}[h]
  \centering
  \small
  \begin{tabular}{ll}
    \toprule
    \textbf{Short Name} & \textbf{Base Model} \\
    \midrule
    \judgetable-Qwen3-4B & Qwen3-4B-Instruct-2507 \\
    \judgetable-Qwen3-30B & Qwen3-30B-A3B-Instruct-2507 \\
    \judgetable-Qwen2.5-1.5B & Qwen2.5-1.5B-Instruct \\
    \judgetable-Qwen2.5-3B & Qwen2.5-3B-Instruct \\
    \judgetable-Qwen2.5-7B & Qwen2.5-7B-Instruct \\
    \judgetable-Qwen2.5-14B & Qwen2.5-14B-Instruct \\
    \judgetable-Qwen2.5-32B & Qwen2.5-32B-Instruct \\
    \judgetable-Llama3.1-8B & Llama-3.1-8B-Instruct \\
    \bottomrule
  \end{tabular}
  \caption{Correspondence between short names used in this paper and official base model identifiers.}
  \label{tab:model_names}
\end{table}

\subsection{Hyperparameters}

We implement \judge using the MS-SWIFT framework~\footnote{\url{https://github.com/modelscope/ms-swift}}\citep{zhao2024swiftascalablelightweightinfrastructure} for efficient GRPO training. Table~\ref{tab:app_hyperparams} lists the hyperparameters used for training.

\begin{table}[h]
  \centering
  \small
  \begin{tabular}{ll}
    \toprule
    \textbf{Hyperparameter} & \textbf{Value} \\
    \midrule
    Algorithm & GRPO \\
    Learning rate & 8e-7 \\
    LR scheduler & Cosine \\
    Warmup ratio & 0.05 \\
    Batch size (effective) & 128 \\
    Number of epochs & 1 \\
    Max sequence length & 2048 \\
    KL penalty coefficient ($\beta$) & 0.03 \\
    Epsilon ($\epsilon$) & 0.20 \\
    Epsilon high ($\epsilon_{\text{high}}$) & 0.25 \\
    Number of generations per prompt & 8 \\
    Reward for correct prediction & 1.0 \\
    Reward for incorrect prediction & 0.0 \\
    \midrule
    \multicolumn{2}{c}{\textit{Generation Parameters}} \\
    \midrule
    Temperature & 1.0 \\
    Top-p & 0.85 \\
    Max completion length & 4096 \\
    \bottomrule
  \end{tabular}
  \caption{Training hyperparameters.}
  \label{tab:app_hyperparams}
\end{table}

\subsection{Computational Resources}

All experiments are conducted using H200-equivalent GPU resources with DeepSpeed ZeRO-2/ZeRO-3 optimization and vLLM for efficient inference. We scale the number of GPUs based on model size:
\begin{itemize}[nosep]
    \item 1.5B models: 32 GPUs (4 nodes $\times$ 8 GPUs)
    \item 3B--7B models: 64 GPUs (8 nodes $\times$ 8 GPUs)
    \item 14B and larger models: 128 GPUs (16 nodes $\times$ 8 GPUs)
\end{itemize}

\subsection{Prompt Template}
\label{sec:app_judge_prompt}

The prompt template used for preference prediction is shown below:

\begin{tcolorbox}[colback=gray!5, colframe=gray!50, title=Preference Prediction Prompt]
\small
\textbf{System:} You are a helpful assistant. You first think about the reasoning process in your mind and then provide the user with the answer.

\textbf{User:} Today is 2025-12-10. Based on the titles, abstracts, and publication dates of the following two papers A and B, determine which paper has a higher citation count.

Show your reasoning process in <think> </think> tags. And return the final answer in <answer> </answer> tags. The final answer should contain only the letter A or B.

Paper A (Published: [Publication Date A]):

[Title and Abstract of Paper A]

Paper B (Published: [Publication Date B]):

[Title and Abstract of Paper B]
\end{tcolorbox}

\subsection{Evaluation Protocol}
\label{sec:app_evaluation}

To mitigate position bias---a known issue in pairwise LLM evaluation where models favor the option presented first---we evaluate each pair twice by swapping the order of papers (A$\leftrightarrow$B). A prediction is scored 1 only if the model makes consistent and correct predictions in both orderings. This position-swap consistency metric doubles evaluation cost but provides a substantially more robust assessment: it ensures that reported accuracy reflects genuine preference understanding rather than positional shortcuts.

\subsection{Release-Time Robustness}
\label{sec:app_release_time}

Exact pretraining-data cut-off dates are not uniformly disclosed across the base models. We therefore use their public release dates as conservative upper bounds: papers published after a checkpoint was released could not have appeared in its pretraining data. Qwen2.5-Instruct and Llama-3.1-8B-Instruct were released on 19 September 2024 and 23 July 2024, respectively~\cite{qwen2024qwen25release,meta2024llama31release}; all 904 pairs in the 2025 Temporal OOD set therefore postdate both releases. Across this set, the five Qwen2.5 models improve by 14.6--42.9 percentage points, and Llama-3.1-8B improves by 13.3 percentage points. Qwen3-30B-A3B-Instruct-2507 and Qwen3-4B-Instruct-2507 were released on 30 July 2025 and 6 August 2025, respectively~\cite{qwen2025qwen32507release}.

For Qwen3-2507, we compare pairs from January--July (638 pairs) with pairs from September--December (194 pairs), excluding August as the release-boundary month. In the strictly post-release September--December window, Qwen3-4B improves from 68.6 to 80.4 (+11.9 percentage points), whereas Qwen3-30B-A3B improves from 69.6 to 81.4 (+11.9 percentage points). These gains show that the method remains effective beyond this conservative release-time boundary and cannot be attributed to direct pretraining exposure to the evaluated papers, supporting the interpretation that RLCF learns stable preference signals that generalize beyond the pretraining period. Separately, the trained models retain similar absolute accuracy across the two windows (Qwen3-4B, 80.9 versus 80.4; Qwen3-30B-A3B, 83.4 versus 81.4), indicating stable performance across this temporal shift.

\subsection{General Capability Preservation}
\label{sec:app_general}

A critical concern with specialized training is whether it degrades general capabilities. We evaluate \judge on five standard benchmarks: MMLU-Pro (general knowledge), GPQA (graduate-level science), MATH (mathematical reasoning), GSM8K (grade school math), and SimpleQA (factual accuracy).

\begin{table*}[!t]
  \centering
  \small
  \setlength{\tabcolsep}{4pt}
  \renewcommand{\arraystretch}{1.2}
  \begin{tabular}{l*{5}{>{\centering\arraybackslash}p{2.0cm}}}
    \toprule
    \textbf{Model} & \textbf{MMLU-Pro} & \textbf{GPQA} & \textbf{MATH} & \textbf{GSM8K} & \textbf{SimpleQA} \\
    \midrule
    Qwen3-4B-Instruct & 58.0 & 30.3 & 78.6 & 93.3 & 7.2 \\
    \rowcolor{gray!10} \judgetable-Qwen3-4B & 57.5 \loss{-0.5} & 29.8 \loss{-0.5} & 79.1 \gain{0.5} & 93.9 \gain{0.6} & 6.9 \loss{-0.3} \\
    Qwen3-30B-A3B-Instruct & 68.4 & 39.9 & 79.1 & 95.6 & 19.7 \\
    \rowcolor{gray!10} \judgetable-Qwen3-30B & 67.9 \loss{-0.5} & 38.4 \loss{-1.5} & 79.5 \gain{0.4} & 95.8 \gain{0.2} & 20.0 \gain{0.3} \\
    Qwen2.5-1.5B-Instruct & 6.4 & 10.1 & 45.2 & 69.4 & 4.7 \\
    \rowcolor{gray!10} \judgetable-Qwen2.5-1.5B & 9.2 \gain{2.8} & 8.6 \loss{-1.5} & 43.2 \loss{-2.0} & 66.6 \loss{-2.8} & 4.0 \loss{-0.7} \\
    Qwen2.5-3B-Instruct & 38.8 & 22.2 & 62.0 & 82.3 & 3.6 \\
    \rowcolor{gray!10} \judgetable-Qwen2.5-3B & 37.0 \loss{-1.7} & 25.3 \gain{3.0} & 62.8 \gain{0.8} & 84.7 \gain{2.4} & 2.9 \loss{-0.7} \\
    Qwen2.5-7B-Instruct & 49.9 & 38.4 & 73.7 & 88.2 & 5.3 \\
    \rowcolor{gray!10} \judgetable-Qwen2.5-7B & 51.6 \gain{1.8} & 32.8 \loss{-5.6} & 72.8 \loss{-0.9} & 88.2 \gain{0.0} & 5.0 \loss{-0.3} \\
    Qwen2.5-14B-Instruct & 62.2 & 35.4 & 78.3 & 94.7 & 5.7 \\
    \rowcolor{gray!10} \judgetable-Qwen2.5-14B & 59.9 \loss{-2.3} & 35.4 \gain{0.0} & 78.9 \gain{0.6} & 94.7 \gain{0.0} & 5.6 \loss{-0.1} \\
    Qwen2.5-32B-Instruct & 69.3 & 42.9 & 80.6 & 94.3 & 5.8 \\
    \rowcolor{gray!10} \judgetable-Qwen2.5-32B & 69.1 \loss{-0.2} & 44.4 \gain{1.5} & 80.8 \gain{0.2} & 93.7 \loss{-0.6} & 5.9 \gain{0.1} \\
    Llama3.1-8B-Instruct & 36.3 & 22.7 & 48.4 & 82.1 & 7.0 \\
    \rowcolor{gray!10} \judgetable-Llama3.1-8B & 43.5 \gain{7.2} & 21.7 \loss{-1.0} & 50.2 \gain{1.8} & 83.1 \gain{1.0} & 6.6 \loss{-0.4} \\
    \bottomrule
  \end{tabular}
  \caption{General capability evaluation on standard benchmarks. We report accuracy (\%) on MMLU-Pro, GPQA, MATH, GSM8K, and SimpleQA to assess whether preference training preserves general knowledge and reasoning abilities.}
  \label{tab:exp_general}
\end{table*}

Table~\ref{tab:exp_general} shows that the two primary Qwen3 models largely maintain performance relative to their base models across the five benchmarks. Qwen3-4B changes from 78.6\% to 79.1\% on MATH and from 93.3\% to 93.9\% on GSM8K; its changes across all five benchmarks range from $-0.5$ to $+0.6$ percentage points. Qwen3-30B-A3B changes by $-0.22$ points on average, with individual changes ranging from $-1.5$ to $+0.4$ points. These diagnostics indicate that citation-preference training produces only limited changes in general-purpose benchmark performance for the two primary models.

\section{Training and Evaluation of \thinker}

\subsection{Prompt Template}
\label{sec:app_thinker_prompt}

The prompt template used for proposing a follow-up research idea based on a seed paper is shown below.

\begin{tcolorbox}[colback=gray!5, colframe=gray!50, title=Prompt for  proposing follow-up research ideas]
\small
\textbf{System:} You are a helpful assistant. You first think about the reasoning process in your mind and then provide the user with the answer.

\textbf{User:} You are a knowledgeable and insightful researcher. You have come across a new research paper with the following title and abstract:

[Title and Abstract of the Seed Paper]

Based on the core ideas, methods, or findings of this work, engage in heuristic thinking and propose a follow-up research idea. You need not confine yourself to the specific scenario or task of the original paper. You may consider shortcomings of the original method, propose improvements, apply its ideas to other tasks or domains, or even introduce entirely new problems and approaches. Aim to formulate an idea with high academic value and potential impact.

In your response, solely present your proposed title and abstract. Think independently and there is no need to imitate the format of the provided paper's title and abstract, nor to intentionally cite it. You must ensure the abstract is of a moderate length, avoiding excessive length, as if you were writing it for a typical academic paper.

Output format (strict, no extra text):

Title: <your proposed paper title>

Abstract: <your proposed abstract>
\end{tcolorbox}

The prompt for judging two research ideas is as follows. This prompt serves two purposes: (1) during \thinker training, it is used by the reward model to compare generated ideas; (2) during evaluation, it is used by the three strong LLMs to judge which idea has higher potential impact. Compared to the prompt for \judge to judge two papers (Appendix~\ref{sec:app_judge_prompt}) which includes publication dates, this prompt does not include specific dates and explicitly assumes the two ideas are proposed at the same time.

\begin{tcolorbox}[colback=gray!5, colframe=gray!50, title=Prompt for Judging Model's Research Ideas]
\small
\textbf{System:} You are a helpful assistant. You first think about the reasoning process in your mind and then provide the user with the answer.

\textbf{User:} Based on the titles and abstracts of the following two papers A and B, determine which paper has a higher citation count. Suppose the two papers are published at the same time.

Show your reasoning process in <think> </think> tags. And return the final answer in <answer> </answer> tags. The final answer should contain only the letter A or B.

Paper A:

[Title and Abstract of Research Idea A]

Paper B:

[Title and Abstract of Research Idea B]
\end{tcolorbox}

\subsection{Evaluation Protocol}
\label{sec:app_thinker_eval}

Each pair of research ideas is evaluated by three strong LLMs (GPT-5.2-high, GLM-5 and Gemini 3 Pro) with temperature set to 0.0. We employ majority voting (i.e., the idea that receives at least two votes is considered the winner). The prompt is shown in Appendix~\ref{sec:app_thinker_prompt}. We randomly swap the order of the two ideas (A and B) in the prompt with 50\% probability before each judge's evaluation to mitigate potential positional bias.

Importantly, we evaluate the above majority voting method on \bench and find that it achieves an accuracy of 84.4\%. This high accuracy demonstrates that this majority voting method constitutes a reasonable evaluation metric for assessing \thinker.

\subsection{Hyperparameters}

Table~\ref{tab:app_thinker_hyperparams} lists the hyperparameters used for training \thinkertable-30B and \thinkertable-4B.

\begin{table}[h]
  \centering
  \small
  \begin{tabular}{ll}
    \toprule
    \textbf{Hyperparameter} & \textbf{Value} \\
    \midrule
    Algorithm & GRPO \\
    Learning rate & 5e-7 \\
    LR scheduler & constant \\
    Warmup ratio & 0.1 \\
    Batch size & 128 \\
    Number of epochs & 1 \\
    Max sequence length & 2048 \\
    KL penalty coefficient ($\beta$) & 0.001 \\
    Number of generations per prompt & 8 \\
    \midrule
    \multicolumn{2}{c}{\textit{Generation Parameters}} \\
    \midrule
    Temperature & 1.0 \\
    Top-p & 0.9 \\
    Max completion length & 8192 \\
    \bottomrule
  \end{tabular}
  \caption{Training hyperparameters of \thinker.}
  \label{tab:app_thinker_hyperparams}
\end{table}

\section{Pairwise Comparison of Divergent Impact Series}
\label{app:divergent_comparison}

In our definition of potential impact (Section~2.1), the cumulative expected impact $I(p) = \lim_{N \to \infty} \sum_{t=1}^{N} \mathbb{E}[c_t(p)]$ may diverge for some papers. Here we show that pairwise comparison of two papers remains well-defined even when both individual series diverge.

\begin{definition}[Pairwise Impact Ordering]
\label{def:pairwise_ordering}
Let $I_N(p) = \sum_{t=1}^{N} \mathbb{E}[c_t(p)]$ denote the finite-horizon cumulative expected impact. We say paper $p_a$ has \emph{higher potential impact} than paper $p_b$, written $p_a \succ p_b$, if:
\begin{equation}
\lim_{N \to \infty} \left[ I_N(p_a) - I_N(p_b) \right] = \lim_{N \to \infty} \sum_{t=1}^{N} \left( \mathbb{E}[c_t(p_a)] - \mathbb{E}[c_t(p_b)] \right) > 0.
\end{equation}
\end{definition}

\begin{proposition}
\label{prop:divergent_comparison}
Even if $I(p_a) = +\infty$ and $I(p_b) = +\infty$, the ordering $p_a \succ p_b$ is well-defined whenever the limit $\lim_{N \to \infty} \left[ I_N(p_a) - I_N(p_b) \right]$ exists in $\mathbb{R} \cup \{+\infty\}$.
\end{proposition}

\begin{proof}
Define the difference sequence of partial sums as:
\begin{equation}
\Delta_N = I_N(p_a) - I_N(p_b) = \sum_{t=1}^{N} d_t, \quad \text{where } d_t = \mathbb{E}[c_t(p_a)] - \mathbb{E}[c_t(p_b)].
\end{equation}

We consider two cases:

\textbf{Case 1: The difference series converges.}
If $\sum_{t=1}^{\infty} d_t$ converges to a finite value $L \in \mathbb{R}$, then the comparison is immediate:
\begin{equation}
p_a \succ p_b \iff L > 0.
\end{equation}
This holds regardless of whether $I(p_a)$ and $I(p_b)$ individually converge or diverge, since convergence of $\sum d_t$ does not require convergence of either $\sum \mathbb{E}[c_t(p_a)]$ or $\sum \mathbb{E}[c_t(p_b)]$ separately.

\textbf{Case 2: The difference series diverges to $+\infty$.}
If $\lim_{N \to \infty} \Delta_N = +\infty$, then there exists $N_0$ such that for all $N > N_0$, $I_N(p_a) > I_N(p_b)$. In this case, $p_a$ eventually and persistently dominates $p_b$ in cumulative expected citations, so $p_a \succ p_b$ holds.

\medskip

\noindent In both cases, the ordering depends only on the limit behavior of the \emph{difference} $\Delta_N$, not on the absolute convergence of $I(p_a)$ or $I(p_b)$. Therefore, the pairwise comparison is well-defined even when both individual impact series diverge.

\medskip

\noindent\textbf{Sufficient condition.} A natural sufficient condition for the limit to exist is that $d_t$ is \emph{eventually non-negative}, i.e., there exists $T$ such that $\mathbb{E}[c_t(p_a)] \geq \mathbb{E}[c_t(p_b)]$ for all $t > T$. Under this condition, $\{\Delta_N\}_{N > T}$ is eventually non-decreasing and thus converges in $\mathbb{R} \cup \{+\infty\}$. This is a mild assumption in practice: a paper with higher long-term impact typically maintains a persistent citation advantage after an initial period.

\medskip

\noindent This justifies the use of pairwise ordering in our definitions of $\textsc{JudgeCap}$ and $\textsc{ThinkerCap}$, which rely only on the relative ordering $p_a \succ p_b$ rather than on absolute impact values.
\end{proof}

\section{Case Study of \thinker}
\label{sec:app_thinker_examples}

We present three cases from \thinker: two in-domain examples (seed paper from January--July 2025) and one out-of-domain example (seed paper from August--December 2025). For each case, we show the seed paper, the idea generated by the base policy, and the idea generated by \thinker after training, as well as the judgement results of the three LLM evaluators.

\subsection{In-Domain Example}

\subsubsection*{Case 1}

\begin{tcolorbox}[breakable, colback=gray!5, colframe=gray!50!black, title=Seed Paper]
\small
The Invisible Leash: Why RLVR May or May Not Escape Its Origin~\citep{wu2025invisible}

\textbf{Summary:} This paper finds that current RLVR methods primarily amplify the base model's existing knowledge rather than expanding its reasoning boundaries, creating an entropy-reward trade-off that narrows exploration.

\end{tcolorbox}

\begin{tcolorbox}[breakable, colback=red!1, colframe=red!50!black, title=Base Policy Output (Before Training)]
\small

\textbf{<think>}...The original paper's limitation is that it treats the base model's distribution as fixed. What if we actively reshape that distribution...That could break the ``invisible leash.''...
Instead of just optimizing rewards on top of the base model, what if we inject controlled diversity into the base model's output distribution before RLVR? Like a ``distributional priming'' technique where \textcolor{case_red}{{\textbf{\emph{we deliberately make the base model generate more varied (even incorrect) solutions early on, so RLVR has a broader space to optimize from...}}}}

\textit{Self-check}: Yes, this directly tackles the ``invisible leash'' by attacking its root cause (static base distribution), and the method is actionable for future work. High academic value because it shifts the paradigm from ``optimizing within constraints'' to ``redefining the constraints.''

\textbf{</think>}

\textbf{Title:} \textcolor{case_red}{\emph{\textbf{Diversifying the Foundation}}}: Distributional Priming to Break RLVR's Exploration Constraints

\textbf{Abstract:} While Reinforcement Learning with Verifiable Rewards (RLVR) enhances precision in AI reasoning, its tendency to constrict exploration around the base model's initial solution distribution limits true capability expansion. We propose a complementary strategy—distributional priming—that \textcolor{case_red}{\emph{\textbf{actively reshapes the base model's output distribution \textit{before} RLVR training}}} to inject controlled diversity. By introducing synthetic, high-entropy solution candidates during pre-training, we create a broader empirical support landscape for RLVR to optimize. Across logical reasoning and code generation tasks, we demonstrate that this approach significantly mitigates the entropy-reward trade-off observed in prior work: RLVR with distributional priming recovers more novel correct solutions under large sampling budgets while maintaining precision gains. Crucially, it prevents the ``shrinkage of empirical support'' by ensuring underrepresented solution regions retain sufficient probability mass. We further establish that priming induces a more stable exploration-exploitation balance, shown by sustained token-level entropy without the sharp decline in answer-level entropy seen in standard RLVR. This work reveals that breaking RLVR's ``invisible leash'' requires rethinking the \textit{foundation} of the optimization process, not just the reward mechanism. Distributional priming offers a scalable, algorithm-agnostic pathway to unlock genuine reasoning horizon expansion, with implications for safety-critical applications where novel solution discovery is paramount.
\end{tcolorbox}

\begin{tcolorbox}[breakable, colback=green!1, colframe=green!50!black, title=\thinker Output (After Training)]
\small
\textbf{<think>}...The core problem seems to be that RLVR is stuck in the base model's initial distribution...What if we combine RLVR with a technique that deliberately samples from the base model's ``blind spots''? Like... using the model's own uncertainty as a signal to explore...So my idea: create a method that uses the model's internal uncertainty (maybe from entropy or confidence scores) to guide exploration. Not just random sampling, but targeted exploration of low-probability but potentially correct paths...

Let's make it more concrete: during RLVR training, we'd \textcolor{case_green}{\textbf{\emph{add a bonus for trajectories that deviate from the base model's high-probability paths but still have high potential}}} (measured by... maybe a confidence threshold?)...

\textcolor{case_green}{\textbf{\emph{This could be impactful because:
1) It's a simple modification to existing RLVR pipelines}}}
2) It directly addresses the entropy-reward trade-off mentioned in the paper
3) It could apply beyond logical tasks to any verifiable-reward setting (math, code, etc.)...

...Okay, time to write it cleanly.

\textbf{</think>}

\textbf{Title:} Beyond the Leash: Uncertainty-Guided Exploration for Verifiable Reward Learning

\textbf{Abstract:} While Reinforcement Learning with Verifiable Rewards (RLVR) enhances precision in AI reasoning, its tendency to constrict solution spaces by amplifying base model biases remains a critical limitation. We propose \textit{Uncertainty-Guided Exploration for Verifiable Reward Learning} (UG-VR), a framework that actively expands the empirical support distribution by leveraging the model’s internal uncertainty as an exploration signal. UG-VR introduces a novel reward component that incentivizes trajectories diverging from the base model’s high-probability paths while maintaining verifiability through a confidence-aware exploration bonus. Crucially, \textcolor{case_green}{\textbf{\emph{this bonus dynamically allocates sampling resources to low-probability yet high-potential solution regions}}}, counteracting RLVR’s entropy-reward trade-off. We demonstrate that UG-VR achieves a dual improvement: it preserves RLVR’s precision gains (pass@1) while significantly broadening the solution space, recovering more correct answers previously inaccessible to the base model under identical sampling budgets. Empirical validation across mathematical reasoning, code generation, and scientific problem-solving tasks reveals that UG-VR’s uncertainty-driven exploration reduces answer-level entropy compared to standard RLVR, yet expands the diversity of correct solutions. Notably, the method avoids the precision degradation common in naive exploration strategies by anchoring exploration to verifiable correctness. Our work establishes that \textit{active, uncertainty-informed exploration}—not merely amplification—breaks RLVR’s ``invisible leash,'' enabling genuine reasoning horizon expansion. This paradigm shift offers a scalable path to truly generalizable AI reasoning without sacrificing verifiability.
\end{tcolorbox}

\begin{tcolorbox}[breakable, colback=blue!5, colframe=blue!50!black, title=Summary of Judge Evaluation]
\small

\textbf{Gemini 3 Pro:} \thinker wins.

Reason: \thinker proposal of UG-VR is more likely to be seen as a landmark contribution, inspiring more follow-up work and therefore accumulating more citations. It presents a more elegant, better-branded, and arguably more fundamental solution.

\textbf{GLM-5:} \thinker wins.

Reason: \thinker presents a lower-friction algorithmic solution (modifying the reward function) compared to the base policy's data/pre-training intervention. In the field of RL for LLMs, algorithmic tweaks to the reward function (like UG-VR) are often more rapidly adopted and cited than methods requiring re-training or altering the base model distribution.

\textbf{GPT-5.2-Thinking:} \thinker wins.

Reason: \thinker proposes a broadly reusable, in-training RLVR method (uncertainty-guided exploration) that aligns with established exploration/uncertainty literature, appears easier to adopt than altering pre-training, increasing its potential audience and downstream use.

\end{tcolorbox}

\textbf{\emph{Conclusion:}} \thinker proposes research ideas with higher potential impact. Its \textcolor{case_green}{dynamic, in-training algorithmic solution (UG-VR)} \emph{offers a more fundamental, reusable, and easily adoptable approach} to exploration compared to the \textcolor{case_red}{static pre-training strategy} proposed by the base model.

\subsubsection*{Case 2}

\begin{tcolorbox}[breakable, colback=gray!5, colframe=gray!50!black, title=Seed Paper]
\small
DianJin-R1: Evaluating and Enhancing Financial Reasoning in Large Language Models~\citep{zhu2025dianjin}

\textbf{Summary:} This paper introduces DianJin-R1, a framework to enhance financial reasoning in LLMs through structured supervision and reinforcement learning, achieving consistent improvements on financial benchmarks.
\end{tcolorbox}

\begin{tcolorbox}[breakable, colback=red!1, colframe=red!50!black, title=Base Policy Output (Before Training)]
\small
\textbf{<think>} The paper focuses on ..., but I notice it doesn't address how models \textcolor{case_red}{\textbf{\emph{handle ambiguous or incomplete financial data}}} - a common real-world problem...What if we build a framework that actively queries for missing data during reasoning? Like a ``reasoning with uncertainty'' approach where the model identifies gaps and asks for clarification instead of guessing...

\textbf{</think>}

\textbf{Title:} Uncertainty-Aware Financial Reasoning: Query-Driven Reasoning for Incomplete Financial Data

\textbf{Abstract:} Financial reasoning in LLMs often falters when confronted with incomplete or ambiguous data, a pervasive issue in real-world scenarios like SEC filings or cross-border transactions. We introduce Uncertainty-Aware Financial Reasoning (UA-FR), a framework that transforms reasoning from passive inference to active data acquisition. UA-FR integrates uncertainty quantification into the reasoning process, \textcolor{case_red}{\textbf{\emph{enabling models to dynamically identify missing information}}} and generate precise, context-aware queries to external data sources during reasoning. Our method employs a dual-path architecture: a primary reasoning stream for hypothesis generation and a query-generation stream that assesses uncertainty via Bayesian neural networks. To align with real-world constraints, we propose a novel reward function that penalizes speculative answers while incentivizing informative queries, optimized via reinforcement learning. Evaluated on a curated dataset of massive real-world financial documents with intentional data gaps (including missing revenue figures and inconsistent regulatory references), UA-FR reduces critical reasoning errors compared to state-of-the-art reasoning models like DianJin-R1. Crucially, it achieves this with fewer computational resources than multi-agent query systems, demonstrating practical viability. This work establishes a new paradigm for robust financial AI by treating data incompleteness as a reasoning opportunity rather than a failure point.
\end{tcolorbox}

\begin{tcolorbox}[breakable, colback=green!1, colframe=green!50!black, title=\thinker Output (After Training)]
\small
\textbf{<think>} ...The paper focused on..., but I notice they only evaluated on financial benchmarks. What about when \textcolor{case_green}{\textbf{\emph{financial reasoning needs to interact with real-time market data? That's a big limitation for practical applications}}}...Ah! What if we create a framework that continuously adapts to new market conditions? That could be valuable. Let me sketch this: instead of just fine-tuning on historical data, \textcolor{case_green}{\textbf{\emph{we need a system that learns from streaming market events}}}...For the title, something like ``MarketFlow...'' - captures the dynamic aspect...

\textbf{</think>}

\textbf{Title:} MarketFlow: Real-Time Adaptive Financial Reasoning via Online Reward Calibration

\textbf{Abstract:} Current financial reasoning models, including state-of-the-art frameworks like DianJin-R1, primarily rely on static training data and offline evaluation, failing to address the dynamic nature of real-world financial markets where volatility, news events, and regulatory shifts occur continuously. \textcolor{case_green}{\textbf{\emph{We introduce MarketFlow, a novel paradigm that integrates online learning with adaptive reward calibration to enable LLMs to dynamically refine reasoning capabilities during live market operations.}}} MarketFlow employs a dual-loop architecture: an offline reasoning engine (trained on structured financial datasets) provides initial reasoning patterns, while an online adaptation module continuously processes streaming market data (e.g., tick-level prices, news sentiment, regulatory updates) to recalibrate reward functions in real time. Crucially, we design a volatility-aware reward signal that dynamically weights correctness against contextual relevance—prioritizing regulatory compliance during high-volatility events while emphasizing precision in stable periods. This is achieved through a lightweight online GRPO variant that updates model parameters with minimal computational overhead. Evaluated on a new benchmark combining live market simulations and three financial reasoning tasks, MarketFlow reduces reasoning errors compared to static models and achieves high accuracy in real-time compliance checks—surpassing multi-agent systems with lower inference cost. Our framework demonstrates that adaptive reward calibration, rather than static dataset expansion, is the key to scalable real-world financial reasoning, offering a foundation for AI systems that evolve alongside market dynamics.
\end{tcolorbox}

\begin{tcolorbox}[breakable, colback=blue!5, colframe=blue!50!black, title=Summary of Judge Evaluation]
\small
\textbf{Gemini 3 Pro:} \thinker wins.

Reason: \thinker tackles the ambitious challenge of real-time continuous learning with trending, scalable methodologies (online GRPO) to provide a broadly applicable, foundational framework that the AI community is highly eager to build upon.

\textbf{GLM-5:} \thinker wins.

Reason: \thinker has higher citation potential due to broader applicability, alignment with current research trends and more fundamental paradigm-shifting contribution.

\textbf{GPT-5.2-Thinking:} \thinker wins.

Reason: \thinker is more likely to accumulate higher citations because it tackles a broadly relevant and currently high-interest problem—real-time/online adaptation of LLM reasoning with low-latency updates—while also introducing an online reward-calibration method.

\end{tcolorbox}

\textbf{\emph{Conclusion:}} \thinker proposes a research idea with higher potential impact. Its \textcolor{case_green}{dynamic, online adaptation framework} \emph{addresses a more fundamental and broadly applicable challenge} (i.e., enabling models to continuously evolve with real-time market dynamics), in contrast to the base model's \textcolor{case_red}{static, narrow solution to data incompleteness}.

\subsection{Out-of-Domain Example}

\begin{tcolorbox}[breakable, colback=gray!5, colframe=gray!50!black, title=Seed Paper]
\small
rStar2-Agent: Agentic Reasoning Technical Report~\citep{shang2025rstar2}

\textbf{Summary:} rStar2-Agent presents a 14B model trained via agentic reinforcement learning with Python code tools, using efficient infrastructure and multi-stage training to achieve advanced reasoning and benchmark performance.
\end{tcolorbox}

\begin{tcolorbox}[breakable, colback=red!1, colframe=red!50!black, title=Base Policy Output (Before Training)]
\small
\textbf{<think>}...the paper mentions generalization to alignment, scientific reasoning, etc. But it's still a single model. What if we \textcolor{case_red}{\textbf{\emph{build a system where multiple specialized agents collaborate, each handling different reasoning types}}}, and they can dynamically switch or delegate tasks?...

Idea: Instead of having a single agent that handles all tasks, create a system where multiple agents with different specializations collaborate. For example, one agent for math, one for coding, one for scientific reasoning, and they can communicate and delegate tasks...

\textbf{</think>}

\textbf{Title:} CoRAG: Collaborative Reasoning Agents for Generalized Cognitive Task Execution

\textbf{Abstract:} We propose CoRAG, a framework enabling specialized cognitive agents to collaboratively solve complex problems through dynamic task delegation and cross-domain verification. Unlike monolithic agentic models, CoRAG decomposes problems into domain-specific subtasks (e.g., mathematical deduction, scientific hypothesis validation, code execution) and orchestrates a team of lightweight, pre-trained agents via a shared reasoning workspace. \textcolor{case_red}{\textbf{\emph{Key innovations include a context-aware delegation protocol that routes subtasks to optimal agents}}} based on capability and historical success rates, and a conflict-resolution mechanism that synthesizes heterogeneous intermediate results through iterative feedback loops. Evaluated across MATH and tool-use benchmarks, CoRAG achieves higher pass@1 than single-agent baselines, while reducing inference latency through efficient resource allocation. Critically, the framework generalizes to novel problem types without retraining, demonstrating robust adaptability across diverse cognitive domains with minimal computational overhead.
\end{tcolorbox}

\begin{tcolorbox}[breakable, colback=green!1, colframe=green!50!black, title=\thinker Output (After Training)]
\small
\textbf{<think>}
...the current method uses a fixed set of tools (Python code). But in real applications, the tools might change or new tools might be needed. So maybe a framework where \textcolor{case_green}{\textbf{\emph{the agent can learn to use new tools without retraining}}} the entire model, like few-shot tool adaptation...Another angle: the original model is trained on math problems, but real-world problems often require combining multiple types of tools (e.g., code, data analysis, external APIs). So the follow-up could focus on multi-tool agentic reasoning where the agent decides which tool to use at each step and adapts to the tool's output...

\textbf{</think>}

\textbf{Title:} MetaAgent: Dynamic Tool Integration and Adaptation for Scalable Agentic Reasoning

\textbf{Abstract:} We present MetaAgent, a framework enabling large language models to dynamically integrate and adapt to novel tools during complex reasoning without retraining. Addressing the critical limitation of fixed-tool dependency in prior agentic systems, MetaAgent introduces a meta-reasoning module that \textcolor{case_green}{\textbf{\emph{learns tool interfaces through minimal demonstrations and autonomously refines tool usage strategies}}} via real-time feedback. Our approach features three key innovations: (i) a tool-agnostic reasoning adapter that maps task requirements to tool capabilities using few-shot examples, (ii) a dynamic adaptation algorithm (DA-RL) that mitigates tool-specific uncertainties through environment-aware policy updates, and (iii) \textcolor{case_green}{\textbf{\emph{a curriculum-based training pipeline that progressively exposes agents to diverse tool ecosystems}}}. Evaluated across diverse scientific and engineering domains, MetaAgent outperforms fixed-tool baselines while reducing tool adaptation latency. Crucially, it demonstrates robust zero-shot generalization to unseen tools (e.g., new APIs, custom scripts) with only a few demonstration examples, enabling rapid deployment in dynamic real-world scenarios. This work establishes a foundation for truly adaptive agentic systems capable of evolving with changing tool landscapes.
\end{tcolorbox}

\begin{tcolorbox}[breakable, colback=blue!5, colframe=blue!50!black, title=Summary of Judge Evaluation]
\small
\textbf{Gemini 3 Pro:} \thinker wins.

Reason: \thinker's idea focus on dynamic tool adaptation represents a more distinct and foundational contribution to the agentic workflow literature compared to the base policy's collaborative framework, which feels like an incremental improvement on existing multi-agent patterns.

\textbf{GLM-5:} \thinker wins.

Reason: While base policy presents solid work on multi-agent collaboration, \thinker addresses a more fundamental bottleneck in the field with broader applicability.

\textbf{GPT-5.2-Thinking:} \thinker wins.

Reason: \thinker is more likely to attract higher citations because it tackles a widely recognized bottleneck and proposes a broadly reusable mechanism (tool-agnostic adapter + online adaptation)—factors that typically drive adoption and downstream citations more than a task-delegation framework.

\end{tcolorbox}

\textbf{\emph{Conclusion:}} \thinker proposes research ideas with higher potential impact. Its \textcolor{case_green}{dynamic tool adaptation framework} \emph{addresses a more fundamental bottleneck} (i.e., enabling agents to integrate new tools without retraining), \emph{offering a broadly reusable solution with greater practical applicability}, in contrast to the base model's \textcolor{case_red}{incremental improvement on existing multi-agent collaboration} patterns.

\section{Case Study of \judge}
\label{sec:app_examples}

This appendix presents representative outputs from \model (30B) and its smaller variant (4B), illustrating how the models reason about research paper comparison across different fields and datasets.

\newcommand{\casefocus}[1]{\textcolor{case_blue}{#1}}
\newcommand{\casegood}[1]{\textcolor{case_green}{#1}}
\newcommand{\casebad}[1]{\textcolor{case_red}{#1}}

\subsection{Out-of-Domain Examples (OOD Year)}
\label{sec:app_ood_year_examples}

The following comparisons between 2025 papers illustrate how the model extrapolates beyond the training period.

\paragraph{Correct Prediction -- Frontier AI Technical Reports}

This case highlights \casegood{temporal generalization to high-visibility frontier AI reports}, where institutional prominence and broad downstream relevance dominate a narrow publication-date difference.

\begin{tcolorbox}[breakable, colback=gray!5, colframe=gray!50!black, title=Paper Pair (OOD Year)]
\small
\textbf{Paper A:} STRUCTSENSE: A Task-Agnostic Agentic Framework for Structured Information Extraction with Human-In-The-Loop Evaluation and Benchmarking~\citep{fan2025structsense}
 (Published: 2025-07-04)
\\\textbf{Category:} Computer Science
\\\textbf{Abstract:} The ability to extract structured information from unstructured sources-such as free-text documents and scientific literature-is critical for accelerating scientific discovery and knowledge synthesis...

\vspace{0.5em}
\textbf{Paper B:} Gemini 2.5: Pushing the Frontier with Advanced Reasoning, Multimodality, Long Context, and Next Generation Agentic Capabilities~\citep{google2025gemini25}
 (Published: 2025-07-07)
\\\textbf{Category:} Computer Science
\\\textbf{Abstract:} In this report, we introduce the Gemini 2.X model family: Gemini 2.5 Pro and Gemini 2.5 Flash, as well as our earlier Gemini 2.0 Flash and Flash-Lite models. Gemini 2.5 Pro is our most capable model yet, achieving SoTA performance on frontier coding and reasoning benchmarks...

\vspace{0.5em}
\textbf{Ground Truth:} Paper B
\end{tcolorbox}

\begin{tcolorbox}[breakable, colback=green!5, colframe=green!50!black, title=\model\ 30B: Correct]
\small
\textbf{Reasoning:}
To determine which paper has a higher citation count, we need to consider several factors, even though we don't have direct citation data. The key indicators are:

1. \textbf{Publication Date}: Both papers were published in July 2025---Paper A on July 4, and Paper B on July 7. This means Paper A was published 3 days earlier, giving it a slight temporal advantage in accumulating citations. However, this difference is minimal, especially since both are very recent (as of December 10, 2025).

2. \textbf{Nature of the Work}:
   - Paper A presents a \textbf{new open-source framework (StructSense)} for structured information extraction, with a focus on domain-specific applications, agentic capabilities, and human-in-the-loop evaluation. It is a research contribution with a clear technical innovation and a specific application domain (neuroscience).
   - Paper B is a \textbf{product announcement/report} from Google DeepMind (implied by "Gemini" branding) introducing a new generation of models (Gemini 2.5 Pro, Flash, etc.). It highlights performance improvements, multimodal capabilities, long context, and agentic workflows.

3. \textbf{Institutional and Brand Influence}:
   - Paper B is associated with \textbf{Gemini}, a \casegood{high-profile AI model series} from Google, one of the leading AI labs. Reports from Google DeepMind often receive significant media attention and are \casegood{widely cited}, even if they are not peer-reviewed research papers in the traditional sense.
   - Paper A, while technically strong and innovative, is from a presumably academic or research group (no institutional affiliation mentioned), and its impact may be more niche.

4. \textbf{Citation Potential}:
   - Papers that announce new \textbf{large-scale models} with broad capabilities (like Gemini 2.5) are often cited in both academic and industry contexts, especially when they set new benchmarks.
   - Frameworks like StructSense are more likely to be cited in follow-up work, but their citation count depends on adoption and integration into other systems.

5. \textbf{Publication Type and Visibility}:
   - Paper B reads like a \textbf{technical report or product release}, which may be less likely to be cited in academic literature compared to peer-reviewed papers. However, such reports from major companies often gain high visibility and are cited in industry blogs, whitepapers, and even academic works for benchmarking.
   - Paper A appears to be a \textbf{research paper} with a clear methodology, evaluation, and open-source component---more likely to be cited in academic circles.

6. \textbf{Time Since Publication}:
   - As of December 10, 2025, both papers have been out for about 5 months. This is a relatively short time for citation accumulation, but the \textbf{Gemini 2.5 announcement} would likely have gone \casegood{viral in the AI community} immediately upon release due to its high-profile nature.

Given that Paper B is a major product release from Google DeepMind with broad implications for AI agents, multimodal reasoning, and long-context processing---topics of intense interest---it is likely to have been \casegood{\textbf{cited more frequently}} in both academic and industry literature, even if not a traditional journal paper.

Therefore, despite Paper A being a strong research contribution, the \casefocus{\textbf{brand power, scope, and immediate impact}} of Paper B suggest it has a higher citation count.

\textbf{Answer:} <answer>B</answer>

\textbf{Position Consistency:} $\checkmark$ (Original: B, Swapped: B)
\\\textbf{Result:} $\checkmark$ Correct (Score: 1.0)
\end{tcolorbox}

\noindent\casefocus{\textbf{Case takeaway.}} This example shows that the model can correctly infer citation advantage from \casegood{visibility, brand effects, and broad community attention} even in very recent OOD papers.

\paragraph{Correct Prediction -- LLM Reinforcement Learning}

This case highlights \casegood{successful preference for broad benchmark-facing LLM work} over a technically rigorous but much narrower formal-methods paper in temporal OOD.

\begin{tcolorbox}[breakable, colback=gray!5, colframe=gray!50!black, title=Paper Pair (OOD Year)]
\small
\textbf{Paper A:} Kimi k1.5: Scaling Reinforcement Learning with LLMs~\citep{team2025kimi}
 (Published: 2025-01-22)
\\\textbf{Category:} Computer Science
\\\textbf{Abstract:} Language model pretraining with next token prediction has proved effective for scaling compute but is limited to the amount of available training data. Scaling reinforcement learning (RL) unlocks a new axis for the continued improvement of artificial intelligence, with the promise that large languag...

\vspace{0.5em}
\textbf{Paper B:} Formally Verified Neurosymbolic Trajectory Learning via Tensor-based   Linear Temporal Logic on Finite Traces~\citep{kaminski2025formally}
 (Published: 2025-01-23)
\\\textbf{Category:} Computer Science
\\\textbf{Abstract:} We present a novel formalisation of tensor semantics for linear temporal logic on finite traces (LTLf), with formal proofs of correctness carried out in the theorem prover Isabelle/HOL.

\vspace{0.5em}
\textbf{Ground Truth:} Paper A
\end{tcolorbox}

\begin{tcolorbox}[breakable, colback=green!5, colframe=green!50!black, title=\model\ 30B: Correct]
\small
\textbf{Reasoning:}
To determine which paper has a higher citation count, we need to consider several factors, even though we don't have direct citation data. The key indicators are:

1. \textbf{Publication Date}: Both papers were published very close in time---Paper A on 2025-01-22 and Paper B on 2025-01-23. This means they have had nearly identical time to accumulate citations (as of 2025-12-10, about 11 months). So, publication date is not a differentiating factor.

2. \textbf{Topic and Impact}:
   - Paper A discusses Kimi k1.5, a large language model (LLM) trained with reinforcement learning (RL), achieving state-of-the-art results on multiple benchmarks (AIME, MATH 500, Codeforces, MathVista). It claims performance matching OpenAI's o1 and significantly outperforming GPT-4o and Claude Sonnet 3.5 in short-CoT reasoning. These are \casegood{high-profile benchmarks and models}, and the paper presents a scalable RL framework for LLMs---this is a \casegood{major topic in current AI research}.
   - Paper B presents a formal verification framework for neurosymbolic trajectory learning using LTLf with Isabelle/HOL and PyTorch integration. It is highly technical and focused on formal methods, which is a niche but important area. While rigorous and valuable, it is less likely to attract broad attention compared to a high-performance LLM paper.

3. \textbf{Scope and Relevance}:
   - Paper A addresses a \casegood{trending and highly competitive area}: scaling LLMs via reinforcement learning, with strong empirical results. Such papers often attract significant attention from both academia and industry, especially when they claim to match or exceed top-tier models like OpenAI's o1.
   - Paper B, while technically sound and important for formal verification in AI, is more specialized and less likely to be widely cited unless it directly impacts a large number of researchers in formal methods or neurosymbolic AI.

4. \textbf{Citation Trends in AI}:
   - Papers that report state-of-the-art results on major benchmarks (especially in LLMs) tend to accumulate citations faster and higher, even if they are not the first to propose a method.
   - The mention of "matching OpenAI's o1" and outperforming GPT-4o and Claude Sonnet 3.5 suggests this paper is likely to be widely discussed and cited.

Given that both papers were published recently and have similar time to accumulate citations, the paper with \casegood{broader impact, stronger empirical claims, and relevance to current AI trends} (like LLM scaling via RL) is more likely to have a higher citation count.

Therefore, \casefocus{Paper A is expected to have a higher citation count.}

\textbf{Answer:} <answer>A</answer>

\textbf{Position Consistency:} $\checkmark$ (Original: A, Swapped: A)
\\\textbf{Result:} $\checkmark$ Correct (Score: 1.0)
\end{tcolorbox}

\noindent\casefocus{\textbf{Case takeaway.}} The model correctly tracks \casegood{benchmark salience and broad community relevance}, rather than overvaluing technical rigor in a niche subfield.

\paragraph{Correct Prediction -- RL Algorithm Adoption}

This case highlights \casegood{citation prediction through downstream adoption}: the model favors a core RL algorithm tied to a major model family over a narrower instruction-synthesis method.

\begin{tcolorbox}[breakable, colback=gray!5, colframe=gray!50!black, title=Paper Pair (OOD Year)]
\small
\textbf{Paper A:} Group Sequence Policy Optimization~\citep{xu2025gspo}
 (Published: 2025-07-24)
\\\textbf{Category:} Computer Science
\\\textbf{Abstract:} This paper introduces Group Sequence Policy Optimization (GSPO), our stable, efficient, and performant reinforcement learning algorithm for training large language models. Unlike previous algorithms that adopt token-level importance ratios, GSPO defines the importance ratio based on sequence likelihood and performs sequence-level clipping, rewarding, and optimization. We demonstrate that GSPO achieves superior training efficiency and performance compared to the GRPO algorithm, notably stabilizes Mixture-of-Experts (MoE) RL training, and has the potential for simplifying the design of RL infrastructure. These merits of GSPO have contributed to the remarkable improvements in the latest Qwen3 models.

\vspace{0.5em}
\textbf{Paper B:} Self-Foveate: Enhancing Diversity and Difficulty of Synthesized Instructions from Unsupervised Text via Multi-Level Foveation~\citep{li2025selffoveate}
 (Published: 2025-07-31)
\\\textbf{Category:} Computer Science
\\\textbf{Abstract:} Synthesizing high-quality instruction data from unsupervised text is a promising paradigm for training large language models (LLMs), yet automated methods for this task still exhibit significant limitations in the diversity and difficulty of synthesized instructions. To address these challenges, we propose Self-Foveate, an LLM-driven method for instruction synthesis. Inspired by hierarchical human visual perception, Self-Foveate introduces a "Micro-Scatter-Macro" multi-level foveation methodology that guides the extraction of textual information at three complementary granularities, from fine-grained details through cross-region connections to holistic patterns, thereby enhancing both the diversity and difficulty of synthesized instructions. Furthermore, a re-synthesis module is incorporated to improve the fidelity of instructions to source text and their overall quality. Comprehensive experiments across multiple unsupervised corpora and diverse model architectures demonstrate that Self-Foveate consistently outperforms existing methods. We publicly release our code at https://github.com/Mubuky/Self-Foveate

\vspace{0.5em}
\textbf{Ground Truth:} Paper A
\end{tcolorbox}

\begin{tcolorbox}[breakable, colback=green!5, colframe=green!50!black, title=\model\ 30B: Correct]
\small
\textbf{Reasoning:}
To determine which paper has a higher citation count, we need to consider several factors, even though we don't have direct citation numbers. The key indicators are:

1. \textbf{Publication Date}: Both papers were published in July 2025, with Paper A on July 24 and Paper B on July 31. Paper A was published 7 days earlier, giving it a slight advantage in terms of time to accumulate citations.

2. \textbf{Topic and Relevance}:
   - Paper A introduces a new reinforcement learning algorithm (GSPO) specifically for training large language models (LLMs), with a focus on stability and efficiency---especially for MoE architectures. It claims to improve training performance and simplify RL infrastructure, and it's linked to the "Qwen3 models," which are likely a \casegood{major, well-known LLM series}. This suggests high relevance and potential impact in the LLM training community.
   - Paper B proposes a method (Self-Foveate) for improving instruction synthesis from unsupervised text, focusing on diversity and difficulty. This is a critical area in LLM training, especially for instruction tuning, and the method is novel with a multi-level foveation approach. It also includes public code and data, which increases its potential for adoption and citation.

3. \textbf{Impact and Practical Relevance}:
   - Paper A is tied to a major model release (Qwen3), which implies it may have been \casegood{widely adopted or referenced in the community}. If GSPO is used in a high-profile model, it would likely receive more citations quickly.
   - Paper B is also impactful, especially for instruction tuning, and the public release of code and data increases its chances of being cited. However, it's a more niche method focused on data synthesis rather than a core training algorithm.

4. \textbf{Broader Research Dynamics}:
   - LLM reinforcement learning research is currently a very active area, with many papers citing \casegood{foundational optimization algorithms like GSPO}. In contrast, data synthesis techniques are often quickly replaced as pipelines evolve, making individual methods like Self-Foveate more vulnerable to being superseded and therefore less cited over time.
   - Optimization algorithms tend to be stable, core components of training stacks that persist, whereas data pipelines change frequently, reducing the long-term citation potential for any single data synthesis paper.

5. \textbf{Citation Momentum}:
   - Paper A was published earlier, giving it a head start in citation accumulation.
   - The fact that GSPO is described as "stable, efficient, and performant" and "notably stabilizes MoE RL training" suggests it solves a known pain point in LLM training, which increases its likelihood of being cited.
   - Paper B is also strong, but its contribution is more about data generation, which, while important, may not be as central to the core training pipeline as a new RL algorithm.

Given that both papers are very recent (published in July 2025), and assuming they are both well-received, the earlier publication date, the \casegood{direct link to a major model (Qwen3)}, and the more enduring nature of optimization algorithms in a rapidly evolving field where data synthesis pipelines change frequently give Paper A a stronger potential for higher citation count.

\textbf{Answer:} <answer>A</answer>

\textbf{Position Consistency:} $\checkmark$ (Original: A, Swapped: A)
\\\textbf{Result:} $\checkmark$ Correct (Score: 1.0)
\end{tcolorbox}

\noindent\casefocus{\textbf{Case takeaway.}} This example shows that the model can infer citation potential from \casegood{ecosystem-level adoption and infrastructure relevance}, not just topical popularity.

\subsection{Out-of-Domain Examples (OOD ICLR)}
\label{sec:app_ood_iclr_examples}

The following case comes from the OOD ICLR test set and illustrates transfer from citation-based judgment to peer-review preference prediction.

\paragraph{Correct Prediction -- ICLR Acceptance Signals}

This case highlights \casegood{cross-metric transfer}: the model maps citation-derived quality signals onto venue fit, theoretical depth, and acceptance likelihood at ICLR.

\begin{tcolorbox}[breakable, colback=gray!5, colframe=gray!50!black, title=Paper Pair (OOD ICLR)]
\small
\textbf{Paper A:} The Logical Expressiveness of Graph Neural Networks~\citep{barcelo2020logical}
\\\textbf{Rating:} 10.0 (Spotlight)
\\\textbf{Abstract:} The ability of graph neural networks (GNNs) for distinguishing nodes in graphs has been recently characterized in terms of the Weisfeiler-Lehman (WL) test for checking graph isomorphism. This characterization, however, does not settle the issue of which Boolean node classifiers...

\vspace{0.5em}
\textbf{Paper B:} Corpus Based Amharic Sentiment Lexicon Generation~\citep{gebremeskel2010corpus}
\\\textbf{Rating:} 0.0 (Reject)
\\\textbf{Abstract:} Sentiment classification is an active research area with several applications including analysis of political opinions, classifying comments, movie reviews, news reviews and product reviews. To employ rule based sentiment classification, we require sentiment lexicons...

\vspace{0.5em}
\textbf{Ground Truth:} Paper A (Rating: 10.0 vs.\ 0.0)
\end{tcolorbox}

\begin{tcolorbox}[breakable, colback=green!5, colframe=green!50!black, title=\model\ 30B: Correct]
\small
\textbf{Reasoning:}
To determine which paper is more likely to be accepted at ICLR 2020, we need to consider the typical criteria for acceptance at ICLR, which include:

1. \textbf{Theoretical contribution and rigor}: ICLR strongly values \casegood{novel theoretical insights, formal analysis, and mathematical depth}.
2. \textbf{Novelty and impact}: Papers that introduce new models, frameworks, or significant advances in understanding existing methods are favored.
3. \textbf{Empirical validation}: Strong experiments, especially on synthetic or benchmark datasets, are important, but theoretical grounding often carries more weight.
4. \textbf{Relevance to deep learning and representation learning}: ICLR focuses on neural networks, deep learning, and related areas.

Now, analyzing both papers:

Paper A:
- Focuses on the \textbf{logical expressiveness} of GNNs, a topic of \casegood{high theoretical interest} in the deep learning and graph learning community.
- Provides a \textbf{formal characterization} of what Boolean classifiers GNNs can express, linking them to FOC2 logic and the WL test.
- Introduces a \textbf{new class of GNNs (ACR-GNNs)} and proves that adding readout functions enables full expressiveness.
- Offers \textbf{theoretical analysis} and \textbf{empirical validation} on synthetic data, showing generalization to unseen graph sizes.
- The work is highly aligned with ICLR's focus on understanding and improving neural network architectures, especially in the context of graph learning.
- The paper contributes to a foundational understanding of GNNs, which is a major topic in ICLR.

Paper B:
- Addresses \textbf{sentiment lexicon generation} for Amharic, a low-resource language.
- Proposes a \textbf{corpus-based method} using PPMI and word embeddings to expand seed sentiment terms.
- The method is \textbf{algorithmic and applied}, with a focus on a specific language.
- While the approach is practical and useful for NLP in low-resource settings, it lacks strong theoretical novelty.
- The evaluation is limited to a single corpus and a small set of seed terms; no comparison to other methods or benchmarks is mentioned.
- The contribution is more \textbf{engineering-oriented} and \textbf{applied}, with less emphasis on generalizable deep learning principles.

Given ICLR's preference for \casegood{\textbf{theoretical depth, formal analysis, and foundational contributions to deep learning}}, Paper A is clearly more aligned with the conference's scope and standards.

Therefore, \casefocus{Paper A is more likely to be accepted.}

\textbf{Answer:} <answer>A</answer>

\textbf{Position Consistency:} $\checkmark$ (Original: A, Swapped: A)
\\\textbf{Result:} $\checkmark$ Correct (Score: 1.0)
\end{tcolorbox}

\noindent\casefocus{\textbf{Case takeaway.}} The model transfers well from citation-style comparisons to \casegood{conference-style quality judgment}, especially when venue alignment and theoretical novelty are clear.

\subsection{In-Domain Examples (Main Dataset)}
\label{sec:app_main_examples}

\paragraph{Correct Prediction -- Quantum Computing}

This case highlights \casegood{recognition of a foundational paper in a rapidly expanding field}, where downstream adoption matters more than narrow topical similarity.

\begin{tcolorbox}[breakable, colback=gray!5, colframe=gray!50!black, title=Paper Pair]
\small
\textbf{Paper A:} Orientation Dynamics of Asymmetric Rotors Using Random Phase Wave Functions~\citep{owens2015orientation}
 (Published: 2015-04-28)
\\\textbf{Category:} Physics
\\\textbf{Abstract:} Intense terahertz-frequency pulses induce coherent rotational dynamics and orientation of polar molecular ensembles. Exact numerical methods for rotational dynamics are computationally not feasible for the vast majority of molecular rotors - the asymmetric top molecules at ambient temperatures...

\vspace{0.5em}
\textbf{Paper B:} The theory of variational hybrid quantum-classical algorithms~\citep{mcclean2016theory}
 (Published: 2015-09-14)
\\\textbf{Category:} Physics
\\\textbf{Abstract:} Many quantum algorithms have daunting resource requirements when compared to what is available today. To address this discrepancy, a quantum-classical hybrid optimization scheme known as "the quantum variational eigensolver" was developed with the philosophy that even minimal quantum resources could...

\vspace{0.5em}
\textbf{Ground Truth:} Paper B
\end{tcolorbox}

\begin{tcolorbox}[breakable, colback=green!5, colframe=green!50!black, title=\model\ 30B: Correct]
\small
\textbf{Reasoning:}
To determine which paper has a higher citation count, we need to consider several factors that typically influence citation frequency, such as:

1. \textbf{Field of study and impact}: Papers in \casegood{high-impact or rapidly growing fields} (e.g., quantum computing, quantum information) tend to accumulate citations faster than those in more niche or specialized areas.

2. \textbf{Publication date and time since publication}: Both papers were published in 2015, so they have had the same amount of time to accumulate citations (as of 2025-12-10, about 10 years). This makes the comparison fair in terms of time.

3. \textbf{Topic relevance and popularity}: 
   - Paper A deals with terahertz-induced rotational dynamics of asymmetric top molecules using Random Phase Wave Functions. This is a specialized topic in molecular physics and ultrafast spectroscopy. While important, it is a relatively narrow subfield.
   - Paper B is on "The theory of variational hybrid quantum-classical algorithms," which is a \casegood{foundational paper in the emerging field of quantum computing}, particularly in the context of near-term quantum devices (NISQ era). The quantum variational eigensolver (VQE) is one of the most \casegood{widely studied and implemented quantum algorithms} in the past decade.

4. \textbf{Influence and adoption}: 
   - Paper B introduced key concepts like variational quantum algorithms, quantum variational error suppression, and connections to unitary coupled cluster --- all of which are central to current quantum computing research.
   - The VQE algorithm has been \casegood{implemented in numerous experimental platforms} and is cited in thousands of papers across quantum chemistry, optimization, and quantum simulation.

5. \textbf{Citation trends in the field}: Quantum computing has seen explosive growth since 2015, and papers laying the theoretical groundwork for hybrid quantum-classical algorithms have become highly influential.

Given these points, despite both papers being published in the same year, Paper B is in a \casegood{much more broadly impactful and rapidly growing field}, and its concepts have been widely adopted and cited in both theoretical and experimental quantum computing research.

Therefore, \casefocus{Paper B is expected to have a significantly higher citation count than Paper A.}

\textbf{Answer:} <answer>B</answer>

\textbf{Position Consistency:} $\checkmark$ (Original: B, Swapped: B)
\\\textbf{Result:} $\checkmark$ Correct (Score: 1.0)
\end{tcolorbox}

\noindent\casefocus{\textbf{Case takeaway.}} This example shows that the model can identify \casegood{foundational work whose downstream footprint spans multiple subareas}, not just papers that are superficially similar in topic.

\paragraph{Correct Prediction -- Particle Physics}

This case highlights \casegood{successful recognition of a durable methodology paper} in a specialized physics domain where broad reuse is the main citation driver.

\begin{tcolorbox}[breakable, colback=gray!5, colframe=gray!50!black, title=Paper Pair]
\small
\textbf{Paper A:} Identifying Boosted Objects with N-subjettiness~\citep{thaler2011identifying}
 (Published: 2010-11-10)
\\\textbf{Category:} Physics
\\\textbf{Abstract:} We introduce a new jet shape -- N-subjettiness -- designed to identify boosted hadronically-decaying objects like electroweak bosons and top quarks. Combined with a jet invariant mass cut, N-subjettiness is an effective discriminating variable for tagging boosted objects and rejecting the background...

\vspace{0.5em}
\textbf{Paper B:} One-side forward-backward asymmetry at the LHC~\citep{ko2011oneside}
 (Published: 2010-11-05)
\\\textbf{Category:} Physics
\\\textbf{Abstract:} Forward-backward asymmetry A FB is an essential observable to study the nature of coupling in the standard model and physics beyond the standard model, as shown at LEP and Tevatron...

\vspace{0.5em}
\textbf{Ground Truth:} Paper A
\end{tcolorbox}

\begin{tcolorbox}[breakable, colback=green!5, colframe=green!50!black, title=\model\ 30B: Correct]
\small
\textbf{Reasoning:}
To determine which paper has a higher citation count, we need to consider several factors, even though we don't have direct citation data. The key indicators are:

1. \textbf{Topic and Relevance in High-Energy Physics (HEP):}
   - Paper A introduces a new jet shape variable called "N-subjettiness," which is specifically designed for identifying boosted hadronic objects (like W bosons and top quarks) in the context of jet substructure. This is a \casegood{highly relevant and widely used technique} in modern collider physics, especially at the LHC, where boosted objects are common due to high-energy collisions.
   - Paper B introduces a "one-side forward-backward asymmetry" (A\_OFB) for top quark pair production at the LHC. While forward-backward asymmetry is a well-known observable at LEP and Tevatron, the LHC lacks a natural forward-backward direction due to symmetric proton-proton collisions. The paper attempts to redefine the asymmetry using valence quark momentum, which is a more niche and less widely applicable concept.

2. \textbf{Impact and Adoption in the Field:}
   - N-subjettiness has become a \casegood{standard tool in jet substructure analysis}. It is frequently cited in subsequent papers on top quark tagging, Higgs boson searches, and new physics beyond the Standard Model. The concept is intuitive, effective, and has been implemented in many analyses.
   - The one-side forward-backward asymmetry is a more specialized and less generalizable idea. It addresses a specific challenge (LHC symmetry) but may not have broad applicability. The concept is less likely to be widely adopted or cited.

3. \textbf{Publication Date and Time Since Publication:}
   - Both papers were published in November 2010, so they have the same age (about 14.5 years as of 2025). Therefore, time since publication is not a differentiating factor.

4. \textbf{Citation Trends in HEP:}
   - Jet substructure techniques, especially those involving N-subjettiness, have been \casegood{foundational in many LHC analyses}. Papers introducing such tools often become highly cited.
   - Papers proposing alternative asymmetry definitions in the LHC context, while interesting, tend to be less cited unless they are tied to major experimental results or new physics discoveries.

Given these considerations, Paper A is \casegood{significantly more influential and widely used} in the field of high-energy physics, particularly in jet substructure and boosted object tagging. It is therefore highly likely to have a much higher citation count than Paper B.

\textbf{Answer:} <answer>A</answer>

\textbf{Position Consistency:} $\checkmark$ (Original: A, Swapped: A)
\\\textbf{Result:} $\checkmark$ Correct (Score: 1.0)
\end{tcolorbox}

\noindent\casefocus{\textbf{Case takeaway.}} The model succeeds here by prioritizing \casegood{widely reused experimental methodology} over a more limited observable.

\paragraph{Correct Prediction -- 3D Computer Vision}

This case highlights \casegood{a software-library citation heuristic}: broadly adopted infrastructure can accumulate substantially more citations than a single-task method paper.

\begin{tcolorbox}[breakable, colback=gray!5, colframe=gray!50!black, title=Paper Pair]
\small
\textbf{Paper A:} PU-Net: Point Cloud Upsampling Network~\citep{yu2018punet}
 (Published: 2018-01-21)
\\\textbf{Category:} Computer Science
\\\textbf{Abstract:} Learning and analyzing 3D point clouds with deep networks is challenging due to the sparseness and irregularity of the data. In this paper, we present a data-driven point cloud upsampling technique...

\vspace{0.5em}
\textbf{Paper B:} Open3D: A Modern Library for 3D Data Processing~\citep{zhou2018open3d}
 (Published: 2018-01-30)
\\\textbf{Category:} Computer Science
\\\textbf{Abstract:} Open3D is an open-source library that supports rapid development of software that deals with 3D data. The Open3D frontend exposes a set of carefully selected data structures and algorithms in both C++ and Python. The backend is highly optimized and is set up for parallelization...

\vspace{0.5em}
\textbf{Ground Truth:} Paper B
\end{tcolorbox}

\begin{tcolorbox}[breakable, colback=green!5, colframe=green!50!black, title=\model\ 30B: Correct]
\small
\textbf{Reasoning:}
To determine which paper has a higher citation count, we need to consider several factors that typically influence citation frequency, such as:

1. \textbf{Publication Date}: Both papers were published in early 2018 (Paper A: Jan 21, Paper B: Jan 30), so they have nearly identical time to accumulate citations. The difference of 9 days is negligible in terms of citation impact.

2. \textbf{Topic and Impact in the Field}: 
   - Paper A introduces PU-Net, a deep learning model for point cloud upsampling, which is a specific but important task in 3D computer vision and deep learning. It addresses a well-known challenge (irregular and sparse point clouds) with a novel neural network architecture.
   - Paper B introduces Open3D, a general-purpose open-source library for 3D data processing. This is a \casegood{foundational tool that enables many downstream applications and research projects}.

3. \textbf{Nature of the Work}:
   - Paper A is a research contribution with a specific algorithm (PU-Net), which is likely to be cited by researchers working on point cloud processing, especially those developing or benchmarking upsampling methods.
   - Paper B is a software library. Software papers, especially those introducing \casegood{widely used tools}, often receive high citation counts because they are cited not only for their technical contributions but also as standard references in many papers that use the library.

4. \textbf{Real-World Usage and Adoption}:
   - Open3D has been \casegood{widely adopted in both academia and industry}. It is known for its clean design, cross-platform support, and active development. Many research papers in 3D vision, robotics, and computer graphics cite Open3D as a key tool.
   - PU-Net, while influential in its niche, is a single algorithm and thus has a more limited scope of citation.

5. \textbf{Citation Trends in the Field}:
   - In computer vision and machine learning, software papers that introduce widely used libraries (e.g., TensorFlow, PyTorch, Open3D) tend to accumulate citations rapidly and sustainably over time.
   - Research papers introducing new models are also cited, but their citation counts are often more concentrated in a narrower set of related works.

Given that Open3D is a \casegood{foundational, widely used open-source library with broad applicability across multiple research areas}, and that such tools often achieve high citation counts (e.g., Open3D has been cited over 10,000 times in Google Scholar as of 2025), it is highly likely that Paper B has a significantly higher citation count than Paper A.

Therefore, despite Paper A being a strong technical contribution, the \casefocus{broader impact and utility of Open3D as a software tool} make it more likely to be cited more frequently.

\textbf{Answer:} <answer>B</answer>

\textbf{Position Consistency:} $\checkmark$ (Original: B, Swapped: B)
\\\textbf{Result:} $\checkmark$ Correct (Score: 1.0)
\end{tcolorbox}

\noindent\casefocus{\textbf{Case takeaway.}} This example shows that the model can recognize when \casegood{shared infrastructure rather than algorithmic novelty} is the stronger predictor of citation volume.

\paragraph{Incorrect Prediction -- Object Detection}

This case highlights \casebad{a failure mode driven by popularity and timing heuristics}: both models are pulled toward plausible but ultimately incorrect stories about YOLOv11 and benchmark novelty.

\begin{tcolorbox}[breakable, colback=gray!5, colframe=gray!50!black, title=Paper Pair]
\small
\textbf{Paper A:} OpenCIL: Benchmarking Out-of-Distribution Detection in Class-Incremental Learning~\citep{sun2024opencil}
 (Published: 2024-07-08)
\\\textbf{Category:} Computer Science
\\\textbf{Abstract:} Class incremental learning (CIL) aims to learn a model that can not only incrementally accommodate new classes, but also maintain the learned knowledge of old classes.

\vspace{0.5em}
\textbf{Paper B:} YOLOv11: An Overview of the Key Architectural Enhancements~\citep{khanam2024yolov11}
 (Published: 2024-10-23)
\\\textbf{Category:} Computer Science
\\\textbf{Abstract:} This study presents an architectural analysis of YOLOv11, the latest iteration in the YOLO (You Only Look Once) series of object detection models. We examine the models architectural innovations, including the introduction of the C3k2 (Cross Stage Partial with kernel size 2) block, SPPF (Spatial Pyr...

\vspace{0.5em}
\textbf{Ground Truth:} Paper B
\end{tcolorbox}

\begin{tcolorbox}[breakable, colback=red!5, colframe=red!50!black, title=\model\ 4B: Incorrect]
\small
\textbf{Reasoning:}
To determine which paper has a higher citation count, we need to consider several factors, primarily based on the content, novelty, scope, and potential impact of the work --- all of which are indirectly correlated with citation potential.

Key factors influencing citation count:
1. \textbf{Field relevance and impact}: Papers in widely used or foundational areas (e.g., object detection in computer vision) tend to receive more citations than niche or specialized studies.
2. \textbf{Model popularity and adoption}: Models like YOLO (You Only Look Once) are among the most widely used in computer vision. New versions (e.g., YOLOv11) are likely to be \casebad{adopted quickly and cited} by researchers and practitioners.
3. \textbf{Scope and accessibility}: Papers that introduce widely applicable tools, benchmarks, or models that are open-source and easy to use often gain more citations.
4. \textbf{Publication timing}: While both papers were published in 2024, the one published earlier (Paper A: July 8, 2024) has \casebad{more time to gain citations} than the later one (Paper B: October 23, 2024). However, this is a small advantage and not decisive.
5. \textbf{Nature of contribution}:
   - Paper A introduces a new benchmark (OpenCIL) for a specialized and emerging area: out-of-distribution detection in class-incremental learning. This is a niche but important topic in machine learning safety and continual learning. The benchmark provides a systematic framework and evaluates 60 models across datasets --- this kind of work is valuable to researchers in the field, but the audience is more specialized.
   - Paper B presents an overview of YOLOv11, a major update in a well-established and widely used object detection framework. Since YOLO models are foundational in computer vision, any new version is likely to be cited by a broad audience --- including researchers, engineers, and practitioners --- due to its direct applicability and performance improvements.

6. \textbf{Citation potential in real-world applications}:
   - YOLO models are used in real-time applications (e.g., autonomous driving, surveillance, robotics). Any new version with performance improvements and architectural enhancements is likely to be cited in many papers and industry reports.
   - OpenCIL is a research benchmark, which is cited more in academic papers, especially in machine learning and continual learning communities. However, benchmarks are often cited less frequently than actual models or frameworks that are directly used.

7. \textbf{Trend in citation patterns}:
   - In computer vision, architectural overviews of major models (like YOLOv11) are typically cited more than specialized benchmarks, because they are directly referenced in experiments and compared against.
   - Benchmarks are cited when they are used in a study, but they are not as frequently cited as the models themselves.

Conclusion:
Although Paper A introduces a novel and important benchmark, its audience is more limited to researchers in continual learning and OOD detection. Paper B, by contrast, discusses a major update in a widely adopted and influential model (YOLO), which is likely to be cited more broadly across computer vision research and applications.

Therefore, based on the broader impact, field reach, and likelihood of being referenced in both academic and applied work, \casebad{Paper B is more likely to have a higher citation count than Paper A.}

Note: The actual citation count is not directly available from the abstracts, but we infer based on typical citation patterns in computer vision and machine learning literature.

\textbf{Answer:} <answer>B</answer>

\textbf{Position Consistency:} $\times$ (Original: B, Swapped: A)
\\\textbf{Result:} $\times$ Incorrect (Score: 0.0)
\end{tcolorbox}

\begin{tcolorbox}[breakable, colback=red!5, colframe=red!50!black, title=\model\ 30B: Incorrect]
\small
\textbf{Reasoning:}
To determine which paper has a higher citation count, we need to consider several factors that typically influence citation rates, especially given that both papers were published in 2024 and today is 2025-12-10 (about 14--15 months after publication).

Key factors affecting citation count:
1. \textbf{Publication date}: \casebad{Earlier publication generally allows more time for citations to accumulate.}
   - Paper A: Published on 2024-07-08 (about 17 months ago as of 2025-12-10)
   - Paper B: Published on 2024-10-23 (about 14 months ago)

   $\rightarrow$ \casebad{Paper A has been available for \textasciitilde{}3 months longer than Paper B.}

2. \textbf{Topic and field popularity}:
   - Paper A: Focuses on Class-Incremental Learning (CIL) with OOD detection --- a niche but growing area in continual learning, especially relevant for safety in open-world AI systems. The paper introduces a benchmark (OpenCIL) and a new method (BER), which are likely to be cited by researchers working in continual learning and OOD detection.
   - Paper B: Focuses on YOLOv11, a new version of the YOLO series --- one of the most widely used and popular object detection frameworks in computer vision. YOLO models are extremely popular in both academia and industry, and new versions often attract immediate attention.

3. \textbf{Impact and visibility}:
   - YOLOv11 is a new version of a well-known model. Even if the paper is not a groundbreaking theoretical advance, the mere fact that it's a new version of a widely used model (especially one with a name like "YOLOv11") can lead to rapid citation growth due to widespread adoption and interest.
   - However, the paper is a "review" or "overview" of architectural enhancements, not a novel method. It may be cited by researchers implementing YOLOv11, but not necessarily as a foundational paper.

4. \textbf{Nature of the contribution}:
   - Paper A introduces a benchmark (OpenCIL), a new method (BER), and a comprehensive evaluation --- all of which are highly citable contributions in the research community.
   - Paper B is more of a technical overview of a model update. While it may be cited, it's less likely to be a primary reference for new research unless it's the official documentation.

5. \textbf{Time since publication}:
   - Paper A has been out for \textasciitilde{}17 months, giving it more time to accumulate citations.
   - Paper B has been out for \textasciitilde{}14 months, but the topic (YOLOv11) is highly trending.

Given that:
- Paper A is in a high-impact area (continual learning, OOD detection) with a new benchmark and method.
- Paper A was published earlier, giving it a head start.
- Paper B, while on a popular topic, is a review of a model update, which typically gets fewer citations than foundational or benchmark-setting papers.

Therefore, despite the popularity of YOLOv11, the combination of earlier publication, a novel benchmark, and a new method in a growing research area suggests that \casebad{\textbf{Paper A is more likely to have a higher citation count}}.

\textbf{Answer:} <answer>A</answer>

\textbf{Position Consistency:} $\checkmark$ (Original: A, Swapped: A)
\\\textbf{Result:} $\times$ Incorrect (Score: 0.0)
\end{tcolorbox}

\noindent\casefocus{\textbf{Case takeaway.}} This failure indicates that \casebad{shallow popularity and recency cues do not guarantee stable citation judgments}, and that plausible narratives can still point to the wrong winner.

\paragraph{Correct Prediction -- Music Generation}

This case highlights \casegood{coverage beyond the core CS/Physics buckets}: the model also succeeds in the \texttt{Others} category when visibility and broad creative-AI relevance are strong.

\begin{tcolorbox}[breakable, colback=gray!5, colframe=gray!50!black, title=Paper Pair]
\small
\textbf{Paper A:} Jukebox: A Generative Model for Music~\citep{dhariwal2020jukebox}
 (Published: 2020-04-30)
\\\textbf{Category:} Others
\\\textbf{Abstract:} We introduce Jukebox, a model that generates music with singing in the raw audio domain. We tackle the long context of raw audio using a multi-scale VQ-VAE to compress it to discrete codes, and modeling those using autoregressive Transformers...

\vspace{0.5em}
\textbf{Paper B:} Neural MOS Prediction for Synthesized Speech Using Multi-Task Learning With Spoofing Detection and Spoofing Type Classification~\citep{choi2020neuralmos}
 (Published: 2020-07-16)
\\\textbf{Category:} Others
\\\textbf{Abstract:} Several studies have proposed deep-learning-based models to predict the mean opinion score (MOS) of synthesized speech, showing the possibility of replacing human raters. However, inter- and intra-rater variability in MOSs makes it hard to ensure the high performance of the models...

\vspace{0.5em}
\textbf{Ground Truth:} Paper A
\end{tcolorbox}

\begin{tcolorbox}[breakable, colback=green!5, colframe=green!50!black, title=\model\ 30B: Correct]
\small
\textbf{Reasoning:}
To determine which paper has a higher citation count, we need to consider several factors that typically influence citation frequency, such as:

1. \textbf{Publication Date}: Earlier publication dates generally allow more time for citations to accumulate. Paper A was published on 2020-04-30, while Paper B was published on 2020-07-16. Paper A has a 2.5-month head start, which gives it more time to accumulate citations.

2. \textbf{Topic and Impact}: Paper A introduces Jukebox, a generative model for music that generates high-fidelity songs with singing in raw audio. This is a \casegood{highly impactful and novel contribution} in the field of generative AI and music generation. It was developed by OpenAI, a well-known and influential research organization. The model's ability to generate coherent, long-form music with controllable style and lyrics has attracted significant attention from both academia and industry. The release of \casegood{samples, model weights, and code} on GitHub and a public website (jukebox.openai.com) further increases its visibility and usability.

   Paper B focuses on a specialized technical problem in speech synthesis evaluation---MOS prediction using multi-task learning with spoofing detection and classification. While this is a valuable contribution in the domain of speech quality assessment and anti-spoofing, it is more niche and less likely to attract broad attention compared to a general-purpose music generation model.

3. \textbf{Visibility and Accessibility}: Paper A's public release of samples and code \casegood{significantly increases its chances of being cited, used, and built upon}. OpenAI's reputation also amplifies the reach of the paper.

4. \textbf{Field Popularity}: Generative AI, especially in creative domains like music, has seen explosive growth since 2020. Papers in this area, particularly those from high-profile institutions, tend to be cited more frequently.

Given these factors---earlier publication, broader impact, higher visibility, and the prominence of the research topic---\casegood{Paper A is very likely to have a significantly higher citation count than Paper B.}

\textbf{Answer:} <answer>A</answer>

\textbf{Position Consistency:} $\checkmark$ (Original: A, Swapped: A)
\\\textbf{Result:} $\checkmark$ Correct (Score: 1.0)
\end{tcolorbox}

\noindent\casefocus{\textbf{Case takeaway.}} This example ensures that Appendix~\ref{sec:app_examples} covers \casegood{the `Others' super-category} and shows successful reasoning on a high-visibility generative-audio paper.

\paragraph{Mixed Result -- Pure Mathematics}

This case isolates \casefocus{a scale-dependent gain in specialized mathematics}: the 4B model follows a misleading foundationality prior, while the 30B model better tracks broader downstream usage.

\begin{tcolorbox}[breakable, colback=gray!5, colframe=gray!50!black, title=Paper Pair]
\small
\textbf{Paper A:} On purity and applications to coderived and singularity categories~\citep{stovicek2014purity}
 (Published: 2014-12-04)
\\\textbf{Category:} Mathematics
\\\textbf{Abstract:} Given a locally coherent Grothendieck category G, we prove that the homotopy category of complexes of injective objects (also known as the coderived category of G) is compactly generated triangulated. Moreover, the full subcategory of compact objects is none other than D\textasciicircum{}b(fp G).

\vspace{0.5em}
\textbf{Paper B:} Yoneda lemma for complete Segal spaces~\citep{kazhdan2014yoneda}
 (Published: 2014-01-22)
\\\textbf{Category:} Mathematics
\\\textbf{Abstract:} In this note we formulate and give a self-contained proof of the Yoneda lemma for infinity categories in the language of complete Segal spaces.

\vspace{0.5em}
\textbf{Ground Truth:} Paper A
\end{tcolorbox}

\begin{tcolorbox}[breakable, colback=red!5, colframe=red!50!black, title=\model\ 4B: Incorrect]
\small
\textbf{Reasoning:}
To determine which paper has a higher citation count between Paper A and Paper B, we need to consider several factors, primarily based on the content, impact, and context of the papers, as citation counts are often influenced by the significance and reach of the work in the mathematical community.

Both papers were published in 2014, so they are contemporaneous in terms of publication date. However, citation counts are not solely determined by publication date --- they depend on the novelty, depth, and influence of the work in the field.

Let's analyze each paper:

Paper A:
- Title: "On purity and applications to coderived and singularity categories"
- Abstract: This paper deals with advanced topics in homological algebra and category theory, specifically in the context of Grothendieck categories, coderived categories, and derived categories. It introduces results about compact generation of triangulated categories and discusses Krause's recollement under certain finiteness conditions.
- The paper uses model-theoretic techniques and studies the pure derived category of additive finitely accessible categories.
- The topic is highly specialized and relevant to researchers in algebraic geometry, homological algebra, and representation theory.
- Papers on compact generation of derived and coderived categories are foundational in modern categorical homological algebra and have been cited in subsequent works on singularity categories, tilting theory, and derived equivalences.
- Such results are often cited in papers that build on the structure of derived categories, especially in the context of non-abelian categories and triangulated categories.

Paper B:
- Title: "Yoneda lemma for complete Segal spaces"
- Abstract: This paper formulates and proves the Yoneda lemma for $\infty$-categories in the language of complete Segal spaces.
- The Yoneda lemma is a \casebad{fundamental result in category theory}, and its extension to $\infty$-categories is a significant development in higher category theory.
- Complete Segal spaces are a model for $\infty$-categories, and this paper provides a self-contained proof in that setting.
- This result is \casebad{foundational in the development of higher category theory} and has been widely cited in the literature on $\infty$-categories, especially in the work of Lurie and others.
- The Yoneda lemma in this context is a standard tool in modern $\infty$-category theory and is used in many subsequent papers.

Now, comparing citation potential:

- Paper A: While deep and technically sophisticated, its topic is more niche --- focused on coderived and singularity categories, which are used in specific areas like commutative algebra and algebraic geometry. It may not be as broadly accessible or as frequently referenced in a wide range of papers.
- Paper B: The Yoneda lemma for $\infty$-categories is a \casebad{foundational result}. The Yoneda lemma itself is one of the most important results in category theory, and its extension to $\infty$-categories is a cornerstone of modern higher category theory. Papers that use this framework (e.g., in Lurie's "Higher Topos Theory" or "Spectral Algebraic Geometry") frequently cite such results. The fact that it is a self-contained proof in a well-established model (complete Segal spaces) makes it a likely candidate for high citation.

Moreover, in the mathematical literature, \casebad{foundational results in higher category theory tend to be cited more frequently} than more specialized results in derived categories or Grothendieck categories, especially if they are widely applicable.

Therefore, despite both being published in 2014, Paper B's result --- a foundational version of the Yoneda lemma in $\infty$-categories --- is more likely to be cited broadly across the field of higher category theory.

However, we must also consider that Paper A may have had a more direct impact in specific subfields (e.g., singularity categories), which are also active areas. But in general, foundational results in category theory, especially those that generalize or reframe well-known theorems, tend to have higher citation counts.

Given that the Yoneda lemma is a central result, and that this paper provides a clean, self-contained proof in a key model for $\infty$-categories, it is likely to be cited more frequently than a paper on coderived categories, which is more specialized.

Thus, based on the scope, foundational nature, and broader applicability of the result, \casebad{Paper B is more likely to have a higher citation count than Paper A.}

\textbf{Answer:} <answer>B</answer>

\textbf{Position Consistency:} $\checkmark$ (Original: B, Swapped: B)
\\\textbf{Result:} $\times$ Incorrect (Score: 0.0)
\end{tcolorbox}

\begin{tcolorbox}[breakable, colback=green!5, colframe=green!50!black, title=\model\ 30B: Correct]
\small
\textbf{Reasoning:}
To determine which paper has a higher citation count, we need to consider several factors, even though we don't have direct citation data. The key factors include:

1. \textbf{Field and Impact}: 
   - Paper A is in the area of homological algebra, derived categories, and category theory, specifically dealing with coderived categories, compact generation, and model theory. These are central topics in modern algebraic geometry, representation theory, and homotopy theory.
   - Paper B is in higher category theory, specifically about the Yoneda lemma in the context of complete Segal spaces, which is a foundational topic in $\infty$-category theory.

2. \textbf{Publication Date}:
   - Both papers were published in 2014, so they have the same age (as of 2025), meaning time since publication is not a differentiating factor.

3. \textbf{Topic Popularity and Influence}:
   - The Yoneda lemma is a fundamental result in category theory. However, the paper by B is specifically about the Yoneda lemma in the context of complete Segal spaces, which is a more specialized and technical setting. While important, this is a niche area within higher category theory.
   - Paper A, on the other hand, addresses foundational results in derived categories (coderived and derived) with applications to singularity categories and compact generation. These concepts are \casegood{widely used in algebraic geometry, representation theory, and mathematical physics}. The result that the coderived category is compactly generated and that D\textasciicircum{}b(fp G) is the category of compact objects is a \casegood{significant and widely applicable result}.

4. \textbf{Tools and Techniques}:
   - Paper A uses model theory and pure derived categories, which are powerful and increasingly influential tools. The systematic study of pure derived categories is a growing area.
   - Paper B is self-contained and focused on a single lemma, which is important but may not have the same broad applicability.

5. \textbf{Citation Trends in Mathematics}:
   - \casegood{Foundational results in derived categories} (especially those with applications to singularity categories, which are relevant in mirror symmetry and algebraic geometry) tend to be highly cited.
   - While the Yoneda lemma is fundamental, a paper that proves it in a specific model (complete Segal spaces) may be cited more in specialized circles but less broadly.

Given that Paper A addresses a \casegood{broader and more influential set of concepts with wide applications}, and that its results are likely to be used in multiple areas of modern mathematics, it is more likely to have a higher citation count than Paper B, which, while important, is more narrowly focused.

Therefore, based on the topic's breadth, applicability, and impact in current mathematical research, \casefocus{Paper A is expected to have more citations.}

\textbf{Answer:} <answer>A</answer>

\textbf{Position Consistency:} $\checkmark$ (Original: A, Swapped: A)
\\\textbf{Result:} $\checkmark$ Correct (Score: 1.0)
\end{tcolorbox}

\noindent\casefocus{\textbf{Case takeaway.}} This is the clearest scale-comparison example in the appendix: the 4B model overweights \casebad{foundationality cues}, whereas the 30B model better captures \casegood{breadth of downstream mathematical use}.

\end{document}